\documentclass[10pt]{article} %
\usepackage[accepted]{rlc}
\usepackage{amssymb}            %
\usepackage{mathtools}          %
\usepackage{mathrsfs}           %
\usepackage{graphicx}           %
\usepackage{subcaption}         %
\usepackage[space]{grffile}     %
\usepackage{url}                %

\usepackage{amsthm}
\usepackage[capitalize,noabbrev]{cleveref}
\usepackage{mmll} %
\usepackage{wrapfig} %

\let\classAND\AND
\let\AND\relax
\usepackage{algorithmic}

\let\AND\classAND
\AtBeginEnvironment{algorithmic}{\let\AND\algoAND}

\usepackage{colortbl}
\definecolor{Lightgray}{RGB}{235,235,235}
\usepackage{xcolor,colortbl}
\usepackage{multirow}
\usepackage{booktabs,nicematrix}
\usepackage{enumitem}

\definecolor{Gray}{gray}{0.85}
\definecolor{LightCyan}{rgb}{0.88,1,1}
\newcolumntype{a}{>{\columncolor{Gray}}c}
\newcolumntype{b}{>{\columncolor{white}}c}

\def \algname {\text{LSVI-ASE}}
\def \decoup {\cK_{\text{DC}}}

\renewcommand*{\backref}[1]{}
\renewcommand*{\backrefalt}[4]{%
\ifcase #1 %
    No citations.%
\or
    (p. #2.)%
\else
    (pp. #2.)%
\fi}%

\title{More Efficient Randomized Exploration for Reinforcement Learning via Approximate Sampling}

\author{Haque Ishfaq\footnotemark[1]\\
    haque.ishfaq@mail.mcgill.ca \\
    Mila, McGill University
    \And
    Yixin Tan\thanks{Equal contribution}\\
    yixin.tan@duke.edu\\
    Duke University
    \And 
    Yu Yang\\
    Duke University
    \And
    Qingfeng Lan\\
    University of Alberta
    \And 
    Jianfeng Lu\\
    Duke University
    \And
    A. Rupam Mahmood\\
    Amii, University of Alberta
    \And 
    Doina Precup\\
    Mila, McGill University
    \And 
    Pan Xu\\
    Duke University
    }

\begin{document}

\maketitle

\begin{abstract}
Thompson sampling (TS) is one of the most popular exploration techniques in reinforcement learning (RL). However, most TS algorithms with theoretical guarantees are difficult to implement and not generalizable to Deep RL. While the emerging approximate sampling-based exploration schemes are promising, most existing algorithms are specific to linear Markov Decision Processes (MDP) with suboptimal regret bounds, or only use the most basic samplers such as Langevin Monte Carlo. In this work, we propose an algorithmic framework that incorporates different approximate sampling methods with the recently proposed Feel-Good Thompson Sampling (FGTS) approach  \citep{zhang2022feel,dann2021provably}, which was previously known to be computationally intractable in general. When applied to linear MDPs, our regret analysis yields the best known dependency of regret on dimensionality, surpassing existing randomized algorithms. Additionally, we provide explicit sampling complexity for each employed sampler. Empirically, we show that in tasks where deep exploration is necessary, our proposed algorithms that combine FGTS and approximate sampling perform significantly better compared to other strong baselines. On several challenging games from the Atari 57 suite, our algorithms achieve performance that is either better than or on par with other strong baselines from the deep RL literature.

\end{abstract}

\section{Introduction}\label{sec:introduction}

A fundamental problem in reinforcement learning (RL) is balancing exploration-exploitation trade-off. One effective mechanism for addressing this challenge is Thompson Sampling (TS) \citep{thompson1933likelihood, strens2000bayesian, osband2016generalization}, which gained popularity due to its simplicity and strong empirical performance \citep{osband2016deep,osband2018randomized,ishfaq2023provable}. 
While there have been numerous works on TS in both RL theory \citep{osband2013more,russo2019worst,zanette2019frequentist,ishfaq2021randomized,xiongnear} and deep RL literature \citep{osband2016deep, osband2018randomized,fortunato2017noisy,plappert2018parameter,ishfaq2023provable,li2024hyperagent}, 
there remains a substantial gap between algorithms that excel in theoretical properties and those that demonstrate strong empirical performance. This disparity highlights the need for a unified framework that provides the ultimate unification of theory and practice for Thompson sampling.

In particular, heuristic RL algorithms that are motivated by TS \citep{osband2016deep,osband2018randomized,fortunato2017noisy} have shown great empirical potential while often lacking any theoretical guarantee. Existing RL theory works on TS suffer from sub-optimal dimension dependency compared to its upper confidence bound (UCB) counterparts \citep{jin2019provably,zanette2019frequentist, ishfaq2021randomized,ishfaq2023provable}. Recently proposed   Feel-Good Thompson sampling (FGTS) \citep{zhang2022feel,dann2021provably} bypasses this issue by incorporating an optimistic prior term in the posterior distribution of Q function. However, these works fail to provide any computationally tractable sampling procedure from this posterior distribution.

Recently there has been some works that use Langevin Monte Carlo (LMC) \citep{dwaracherla2020langevin,xu2022langevin,ishfaq2023provable,hsu2024randomized} to implement TS which are both provably efficient and practical. However, these works lack generality by confining only to LMC and it remains unclear whether many other advanced approximate sampling methods are compatible with this algorithmic scheme for implementing TS. Moreover, the theoretical analyses of these works are limited to linear MDPs \citep{jin2019provably} while having sub-optimal regret bound. Thus, it is unclear how the sampling error of LMC would affect the theoretical regret guarantee under more general structural assumptions. This shows a clear divergence between the RL theory and the deep RL literature when it comes to TS. To this end, we aim to design an approximate TS based RL algorithm that is generalizable and flexible enough to use different approximate samplers while achieving optimal dependency in the regret bound.

With this aim, we propose several FGTS class of algorithms that incorporates different approximate samplers from the Markov Chain Monte Carlo (MCMC) literature. Unlike previous works that assume exact posterior sampling (e.g.,~\cite{zhang2022feel,dann2021provably,agarwal2022model,agarwal2022non}) by assuming access to unrealistic sampling oracle, we propose practically implementable approximate posterior sampling scheme under FGTS framework using different approximate samplers.  

\subsection{Key Contributions}
We highlight the main contributions of the paper below:

\begin{itemize}[leftmargin=*,nosep]

    \item We present a class of practical and efficient TS based online RL algorithms that prioritizes easy implementation and computational scalability. Concretely, we present practically implementable FGTS style algorithms that are based on approximate samplers from the MCMC literature. Our proposed algorithm allows flexible usage of different approximate samplers such as Langevin Monte Carlo (LMC)~\citep{durmus2019analysis} or Underdamped Langevin Monte Carlo (ULMC)~\citep{chen2014stochastic,cheng2018underdamped} and is easy to implement compared to other state-of-the-art exploration focused deep RL algorithms.

    \item Our main theoretical result provides regret bound under general Markov decision processes and value function classes. Our general analytical framework decomposes the regret bound into two components: 1) the idealistic regret bound assuming exact TS (as addressed in many previous works such as  \cite{agarwal2022non,agarwal2022model}), and 2) the additional term introduced by using approximate samplers. This generalizable and fine-grained analysis allows us to analyze the impact of sampling error for any RL setting with an existing exact-sampling regret bound and known convergence rates for the approximate samplers.
    
    \item When applied to linear MDPs \citep{jin2019provably}, our proposed algorithm achieves a regret bound of $\Tilde{O}(dH^\frac32\sqrt{T})$, where $d$ is the dimension of the feature mapping, $H$ is the planning horizon, and $T$ is the total number of steps. This regret bound has the best known dimension dependency for any randomized and UCB based algorithms.

    \item We provide extensive experiments on both $N$-chain environments \citep{osband2016deep} and challenging Atari games \citep{bellemare2013arcade} that require deep exploration. Our experiments indicate our proposed algorithms perform similarly or better than state-of-the-art exploration algorithms from the deep RL literature.
    
\end{itemize}

\subsection{Additional Related Work}\label{sec:related work}

Randomized least-squares value iteration (RLSVI) based algorithms induce deep exploration by injecting judiciously tuned random noise into the value function \citep{russo2019worst, zanette2019frequentist,ishfaq2021randomized,xiongnear}. \citet{osband2016deep,osband2018randomized} propose deep RL variant of RLSVI wherein they train an ensemble of randomly initialized neural networks and view them as approximate posterior samples of Q functions. Another deep RL variant of RLSVI is Noisy-Net \citep{fortunato2017noisy} that directly injects noise to neural network parameters during the training phase. More recently, \citet{dwaracherla2020langevin, ishfaq2023provable} propose using LMC to perform approximate TS. While \citet{ishfaq2023provable} provides regret bound for their proposed LMC-LSVI algorithm under linear MDP, their regret bound, like other existing randomized exploration algorithms, has sub-optimal dependency on the dimension of the linear MDP. \citet{hsu2024randomized} further extends LMC-LSVI to the cooperative multi-agent reinforcement learning setting. \citet{dann2021provably} proposes conditional posterior sampling based on FGTS~\citep{zhang2022feel} with a regret bound that has optimal dependency on the dimension in linear MDPs, but their algorithm is intractable due to the need to access to some unknown sampler oracle.

\section{Preliminary}\label{sec:preliminary}

In this paper, we consider an episodic discrete-time Markov decision process (MDP), denoted by $(\mathcal{S}, \mathcal{A}, H, \mathbb{P}, r)$, with $\mathcal{S}$ being the state space, and $\mathcal{A}$ the action space. The MDP is non-stationary across $H$ different stages, which form an episode. $H$ is also often referred to as the episode length. $\mathbb{P} = \{\mathbb{P}_h\}_{h=1}^H$ is the collection of the state transition probability distributions,  $\mathbb{P}_h(\cdot \mid x, a)$ denotes the transition kernel at stage $h \in [H]$.  Let $r = \{r_h\}_{h=1}^H$ be the collection of reward functions, which we assume to be deterministic and bounded in $[0,1]$ for the simplicity of the presentation.
 
We define a policy $\pi$ in this MDP as a collection of $H$ functions $\{\pi_h : \mathcal{S} \rightarrow \mathcal{A}\}_{h\in [H]}$, where $\pi_h(x)$ means the action at state $x$ given by policy $\pi$ at stage $h$ of the episode.  

At stage $h \in [H]$, we define the value function $V_h^\pi: \mathcal{S} \rightarrow \mathbb{R}$ as the total expected rewards collected by the agent if it starts at state $x_h=x$ and follows policy $\pi$ onwards, i.e., $V_h^\pi(x) = \mathbb{E}_{\pi, \PP} \big[\sum_{h'=h}^H r_{h'}(x_{h'},a_{h'}) \big| x_h = x\big]$.

We also define the action-value function (or Q function) $Q_h^\pi: \mathcal{S}\times \mathcal{A} \rightarrow \mathbb{R}$ as the total expected rewards by the agent if it starts at state $x_h=x$ and action $a_h=a$ and follows policy $\pi$ onwards, i.e., $Q_h^\pi(x,a) = \mathbb{E}_{\pi,\PP} \big[\sum_{h'=h}^H r_{h'}(x_{h'},a_{h'}) \big| x_h = x, a_h = a\big]$.

To simplify the notation, we denote operator $[\mathbb{P}_h V_{h+1}^\pi](x,a) = \mathbb{E}_{x' \sim \mathbb{P}_h(\cdot | x, a)} V_{h+1}^\pi(x')$. Thus, we write the Bellman equation associated with a policy $\pi$ as
\begin{align}\label{eq:bellman-eq}
    Q_h^\pi(x,a) = (r_h + \mathbb{P}_h V_{h+1}^\pi)(x, a),\qquad
    V_h^\pi(x) = Q_h^\pi(x, \pi_h(x)),\qquad 
    V_{H+1}^\pi(x)  = 0.
\end{align}
We denote the optimal policy as $\pi^*=\{\pi_h^*\}_{h=1}^{H}$, which is the collection of the optimal policies at each stage $h$. We further denote $V_h^*(x) = V_h^{\pi^*}(x)$ and $Q_h^*(x, a) = Q_h^{\pi^*}(x, a)$. It can be shown that $\pi^*$ is a deterministic policy and it satisfies $Q_h^{\pi*}(s,a)=\max_{\pi}Q^{\pi}(s,a)$ and $V_h^{\pi^*}(s)=\max_{\pi}V_h^{\pi}(s)$ for all $s\in\cS$ and $a\in\cA$ \citep{bertsekas2019reinforcement}.

We denote the Bellman optimality operator by $\cT_h^*$ that maps any function $Q$ over $\cS\times\cA$ to 
\begin{align}\label{def:bellman_optimality_operator}
   \textstyle [\cT_h^* Q](x,a) = r_h(x,a) + \EE_{x'\sim \PP_h(\cdot|x,a)}\big[\max_{a'\in \cA}Q_{h+1}(x',a')\big].
\end{align}
Note that the optimal Q function satisfies 
$\cT_h^* Q_{h+1}^* = Q_h^*$, for all $h\in[H]$.

The agent follows the following iterative interaction protocol. 
At the beginning of each episode $k\in[K]$, an adversary picks an initial state $x_1^k$ for stage $1$, and the agent executes a policy $\pi^k$ and updates the policy in the next stage according to the received rewards. 
We measure the suboptimality of the agent by the total regret  defined as
\begin{equation*}
    \text{Regret}(K) = \sum_{k=1}^K \text{REG}_k := \sum_{k=1}^K \big[V_1^*(x_1^k) - V_1^{\pi^k}(x_1^k)\big].
\end{equation*}

\subsection{Notations}
We use $a=O(b)$ to indicate that $a\leq Cb$ for a universal constant $C>0$. Also, we write $a=\Theta(b)$ if there are universal constants $c'>c>0$ such that $cb\leq a\leq c'b$, and the notation $\Tilde{O}(\cdot)$ and $\Tilde{\Theta}(\cdot)$ mean they hide polylog factors in the parameters. For two probability distributions $p$ and $q$ on the same probability space, we denote their total variation (TV) distance by $TV(p, q)$. For $T:\mathbb{R}^d \to \mathbb{R}^d$, 
the pushforward of a distribution $p$ is denoted as $T_\# p$, 
such that 
$T_\# p(A) = p( T^{-1}(A))$ for any measurable set $A$.

\section{Algorithm Design}\label{sec:algorithm design}

In this section, we present our core algorithm, a general framework that leverages Feel-Good Thompson Sampling (FGTS)~\citep{zhang2022feel, dann2021provably} alongside various approximate sampling techniques such as Langevin Monte Carlo (LMC) \citep{durmus2019analysis} and Underdamped Langevin Monte Carlo (ULMC) \citep{chen2014stochastic,cheng2018underdamped}. The proposed general algorithm is displayed in \Cref{Algorithm:FG-LMC}.

Our algorithm design resembles that of \citet{ishfaq2023provable} in the sense that, unlike other approximate TS algorithms \citep{russo2019worst,zanette2019frequentist,ishfaq2021randomized}, it performs exploration by coupling approximate sampling into value iteration step. However, our design choice offers significant flexibility of the algorithm: it allows us to employ a wide range of prior distributions and integrate different samplers, enabling us to tailor the exploration process to specific problem characteristics. The generality of our framework allows it to address suboptimality observed in existing exploration approaches \citep{ishfaq2023provable}. By incorporating a flexible prior selection mechanism, we can overcome limitations inherent in specific prior choices employed by other methods. This flexibility enables us to potentially achieve better performance across diverse exploration problems.

\begin{algorithm}[t]
\caption{ Least-Squares Value Iteration with Approximate Sampling Exploration (\algname)}\label{Algorithm:FG-LMC}
\begin{algorithmic}[1]
\STATE Input: feel-good prior weight $\eta$, step sizes $\{\tau_{k, h} > 0\}_{k, h\geq 1}$,  temperature $\beta$, friction coefficient $\gamma$, loss function $L_k(w)$.
\STATE Initialize $w_h^{1,0} =  \textbf{0}$ for $h \in [H]$, $J_0 = 0$.
\FOR{episode $k=1,2,\ldots, K$}
\STATE Receive the initial state $s_1^k$.

\FOR{step $h=H, H-1,\ldots, 1$} \label{line:begin_inner_for_loop}
    \STATE $w_h^{k,0} = w_h^{k-1, J_{k-1}}$\label{line:warm-start}
    \FOR{$j = 1, \ldots, J_k$}
        \STATE Generate $w_h^{k,j}$ via a sampler in \Cref{sec:different_samplers} \label{line:sampling_w}
    \ENDFOR \label{line:end_inner_for_loop}

\STATE  $Q^k_{h}(\cdot,\cdot) \leftarrow \min\{Q(w_h^{k,J_k};\phi(\cdot, \cdot)) , H -h +1\}^{+}$  \label{algline:q-update} \label{Eq:q-function-estimate-truncated}
\STATE  $V^k_{h}(\cdot) \leftarrow \max_{a \in \mathcal{A}} Q_{h}^{k}(\cdot,a)$. \label{Alg:min-for-Q}
\ENDFOR

\FOR{step $h=1, 2, \ldots, H$}

\STATE Take  action  $a^k_{h} \leftarrow \argmax_{a \in \mathcal{A}} Q_h^{k}(s_h^k,a)$.
\STATE  Observe reward $r^k_{h}(s_h^k,a_h^k)$, get next state $s^k_{h+1}$. 
\ENDFOR
\ENDFOR
\end{algorithmic}
\end{algorithm}

\subsection{Feel-Good Thompson Sampling}\label{sec:feel_good}

Assume we have collected data trajectories in the first $k-1$ episodes as $D_{k-1} = \{(x_1^\tau, a_1^\tau, r(x_1^\tau, a_1^\tau)),\ldots,(x_H^\tau, a_H^\tau, r(x_H^\tau, a_H^\tau))\}_{\tau=1}^{k-1}$. To estimate the $Q$-function for stage $h$ at the $k$-th episode of the learning process, we define the following loss function for $h \in [H]$:
\begin{align}\label{eq:loss_h>0}
    \textstyle L_h^k(w_h) = \eta \sum_{\tau=1}^{k-1}\left[r_h(x_h^\tau, a_h^\tau) + \text{max}_{a\in\mathcal{A}} Q^k_{h+1}(x_{h+1}^\tau, a) - Q(w_h;\phi(x_h^\tau, a_h^\tau))\right]^2, 
\end{align}
where $\phi(\cdot,\cdot)$ is a feature vector of the corresponding state-action pair and $Q(w_h;\phi(x_h^\tau, a_h^\tau))$ denotes any possible approximation of the Q function that is parameterized by $w_h$ and takes $\phi(x_h^\tau, a_h^\tau)$ as input. $Q_h^k$ is defined in Line \ref{Eq:q-function-estimate-truncated} of \Cref{Algorithm:FG-LMC} and is the truncated estimated Q function.
Moreover, we let  $L^k_0(w_1) = -\lambda \text{max}_{a\in\mathcal{A}} Q(w_1; x_1^k, a)$, where $L_0^k$ is the Feel-Good exploration prior term \citep{zhang2022feel}. The posterior distribution at episode $k$ and stage $h> 1$ is then given by
\begin{align*}
    q_k^h(w_h)\propto p_0^h(w_h)\exp(-L_h^k(w_h)),
\end{align*}
where $p_0^h$ is the prior distribution of $w_h$. And at stage $h=1$, we have $q_k^1(w_1)\propto p_0^1(w_1)\exp(-L_1^k(w_1)-L_0^k(w_1))$. Then at episode $k$, the exact target (joint) posterior of $Q$, which is denoted by $q_k$, is given by 
\begin{align}\label{Eq:posterior_of_Q}
   \textstyle  q_k(w) \propto p_0(w)\exp(-\sum_{h=0}^H L_h^k(w_h)),
\end{align}
where $p_0(w) = \prod_{h=1}^H p_0^h(w_h)$. 
Compared to standard TS \citep{thompson1933likelihood,osband2013more}, FGTS incorporates an additional exploration term $\exp(-L_0^k(w_1))$ in the likelihood function. 
This term encourages the selection of value functions at the first time step that yield large values for the initial state. This bias is particularly beneficial during early learning stages when wider exploration is crucial. 

\textbf{Challenges of FGTS in practice:} Despite the regret of FGTS has been proven to be achieving the optimal dependency on the dimension in bandits \citep{zhang2022feel} and reinforcement learning \citep{dann2021provably}, existing FGTS algorithms are often computationally intractable as they assume access to sampling oracles for sampling from a high-dimensional distribution at each iteration. Specifically, previous FGTS based algorithms proposed by \citet{zhang2022feel, dann2021provably} simply sample $Q$ function from the posterior distribution defined in \eqref{Eq:posterior_of_Q} in the beginning of each episode and then follow a greedy policy with respect to the sampled $Q$ function. However, this assumes access to a sampling oracle which allows one to sample from \eqref{Eq:posterior_of_Q} that is not generally available in practice. Since we cannot directly sample from the true posterior distribution $q_k$, we propose using approximating samplers to generate posterior estimates which we describe next. 

\subsection{Approximate Samplers}\label{sec:different_samplers}
In this subsection, we present different approximate sampling methods which we use to approximately sample from the posterior distribution defined in \eqref{Eq:posterior_of_Q}. Let $p\propto e^{-L}$ be a probability density on $\mathbb R^d$ such that $L$ is continuously differential. The goal is to generate samples from $p$.

\textbf{Langevin Monte Carlo. } LMC leverages the Euler discretization method to approximate the continuous-time Langevin diffusion process, making it a popular sampling algorithm in machine learning~\citep{welling2011bayesian}. Langevin diffusion with stationary distribution $p$ is the stochastic process defined by the stochastic differential equation (SDE) $d w_t = -\nabla L(w_t) dt + \sqrt{2}dB_t$, 
where $B_t$ is a standard Brownian motion in $\mathbb R^d$. To obtain the LMC algorithm, we take the Euler-Murayama discretization of the SDE. For a fixed step size $\tau > 0$, temperature $\beta$ and $w_0\in\mathbb{R}^d$, LMC is defined by the iteration
\begin{align*}
    w_{k+1} = w_k - \tau\nabla L(w_k) + \sqrt{2\beta^{-1}\tau}\xi_k,
\end{align*}
where $\xi_k\sim \mathcal{N}(0, I_d)$. Previous works have thoroughly established strong theoretical guarantees for the convergence of LMC~\citep{dalalyan2017theoretical, xu2018global,zou2021faster}.

In Line \ref{line:sampling_w} of  \Cref{Algorithm:FG-LMC}, we can use LMC to approximately sample $w_h^{k,J_k}$ from the posterior defined in \eqref{Eq:posterior_of_Q}. In our deep RL experiment in \Cref{sec:experiments}, we also incorporate adaptive bias term for the gradient of loss function as introduced in \citet{ishfaq2023provable}. 

\textbf{Underdamped Langevin Monte Carlo. } 
While LMC offers an elegant approach, its scalability suffers as the problem dimension, error tolerance, or condition number increases~\citep{zheng2024accelerating, zhang2023improved}. To mitigate these limitations, we exploit Underdamped Langevin Monte Carlo (ULMC), which exhibits enhanced scalability in such high-dimensional or poorly conditioned settings. The appeal of ULMC lies in its connection to Hamiltonian Monte Carlo (HMC)~\citep{Brooks_2011}. Since underdamped Langevin diffusion (ULD) incorporates a Hamiltonian component, its discretization can be viewed as a form of HMC. Notably, HMC has been empirically observed to converge faster to the stationary distribution compared to LMC~\citep{cheng2018underdamped}. Introducing a balance of exploration and exploitation through momentum, the ULD is given by the SDE
\begin{align*}
\begin{split}
dw_t & = P_t dt, \\
dP_t & = -\nabla L(w_t)dt + \gamma P_tdt + \sqrt{2\beta^{-1}\gamma}dB_t,\\ 
\end{split}
\end{align*}
where $\gamma, \beta>0$ are friction coefficient and temperature respectively. We note that instead of using Euler-Maruyama as for LMC, the ULMC can be implemented in the following way:
\begin{align}
    dw_t & = P_t dt,\notag \\
dP_t & = -\nabla L(w_{k\tau})dt + \gamma P_tdt + \sqrt{2\beta^{-1}\gamma}dB_t,\label{ulmc:exact}
\end{align}
for $t\in[k\tau, (k+1)\tau]$, where $\tau>0$ is the step-size. This formulation of ULMC can be
integrated in a closed form~\citep{cheng2018underdamped, zhang2023improved}, and hence our theoretical analysis is based on this scheme. However, obtaining the closed-form solution is computationally expensive in our setting due to the cubic cost $O(d^3)$ of Cholesky decomposition. So, we employ an adapted Euler-Maruyama method in our experiments for efficient numerical integration. Applying Euler-Maruyama with step size $\tau>0$, we obtain the following iteration scheme of $w_k$ and $P_k$:
\begin{align}\label{Eq:ulmc_EM_scheme}
    & w_{k+1} = w_k + \tau P_k,\notag\\
    & P_{k+1} = P_k - \tau \nabla L(w_k) - \gamma \tau P_k + \sqrt{2\beta^{-1}\gamma \tau}\xi_{k},
\end{align}
where $\xi_k\sim \mathcal{N}(0, I_d)$. In practice, to improve the performance, we follow \citet{ishfaq2023provable} to incorporate adaptive bias term to the gradient, which leads to the following update:
\begin{align}\label{Eq:aulmc}
\begin{split}
    & m_{k}=\alpha_1 m_{k-1}+\left(1-\alpha_1\right)\nabla L(w_k) \\
    & v_{k}=\alpha_2 v_{k-1}+\left(1-\alpha_2\right) \nabla L(w_k) \odot \nabla L(w_k)\\
    & P_{k} = (1 - \gamma \tau) P_{k-1} + \tau \left(\nabla L(w_k)+ a m_{k} \oslash \sqrt{v_{k}+\lambda \bf{1}}\right) + \sqrt{2\beta^{-1}\gamma \tau}\xi_{k}\\
    & w_{k+1} = w_k - \tau P_k,
\end{split}
\end{align}
where the hyperparameters $\alpha_1, \alpha_2 \in [0, 1)$ control the exponential decay rates of the moving
averages~\citep{kingma2014adam}.

\section{Theoretical Analysis}\label{sec:theoretical_analysis} 

This section presents the theoretical analysis of our proposed algorithm. We begin by establishing a regret bound for general function classes, shedding light on the impact of sampling error on regret. Subsequently, we focus on linear MDPs \citep{jin2019provably}, providing a detailed analysis of both the regret bound and the corresponding sampling complexity.

\subsection{Regret Bound for General Function Classes}\label{sec:regret_bound_for_general_function_class}
Assume that the agent is given a $Q$-function class $\cQ = \cQ_1 \times \cQ_2 \times \ldots \times \cQ_H$ of functions $Q = \{Q_h\}_{h \in [H]}$ where $Q_h: \cX \times \cA \rightarrow \RR$. 
For any $Q \in \cQ, h \in [H]$ and state-action pair $x,a$, we define the Bellman residual as 
\begin{equation}
    \cE_h(Q;x,a) = \cE(Q_h, Q_{h+1};x,a) = Q_h(x,a) - \cT_h^* Q_{h+1}(x,a).
\end{equation}

We have the following assumptions on the value-function class: 
\begin{assumption}\label{assumption_realizability}
[Realizability]. Assume that $Q^*\in \cQ$.
\end{assumption}

\begin{assumption}\label{assumption_boundedness}[Boundedness] Assume that $\exists b\geq 1$ such that for all $Q\in\cQ$ and $h\in[H]$, $Q_h(x, a) \in [0, b-1]$,
for all $(x, a)\in \cS \times \cA$.
\end{assumption}
\begin{assumption}\label{assumption_completeness}[Completeness] For all $h\in[H]$ and $Q_{h+1}\in\cQ_{h+1}$, there is a $Q_h\in\cQ_h$ such that $Q_h = \cT_h^* Q_{h+1}$.
\end{assumption}
It's important to note that these assumptions are only necessary for general function classes. One can verify that these assumptions are satisfied in some specific settings, such as linear MDPs \citep{jin2019provably} defined in \Cref{subsec:linear_mdp}.

We first introduce two metrics following~\cite{dann2021provably} that characterizes the structural complexity of the MDP and the effective size of the value-function class $\cQ$ respectively.
\begin{definition}\label{def:dc}[Decoupling Coefficient]
Let $\decoup$ be the smallest quantity so that for any sequence of functions $\{Q^k\}_{k\in\mathbb{N}}\subset \cQ$ and $h\geq 0$, it holds that,
{\small
\begin{align*}
     \sum_{h=1}^H \sum_{k=1}^K\EE_{[x_h, a_h]\sim p(\cdot| Q^k, x_1)}\big[\cE_h(Q^k;x_h,a_h)\big] 
    \leq  \inf_{\mu \in(0, 1]}\bigg[ \mu\sum_{h=1}^H \sum_{k=1}^K \sum_{s=1}^{k-1}\EE_{[x_h, a_h]\sim p(\cdot| Q^s, x_1)}[\cE_h(Q^k;x_h,a_h)]^2 +\frac{\decoup}{4\mu}\bigg].
\end{align*}}
\end{definition}

The decoupling coefficient, $\decoup$, measures the growth rate of average Bellman residuals compared to cumulative squared Bellman residuals. We refer the readers to \citet{dann2021provably} for further details on relationship between decoupling coefficient and other complexity measures typically used in RL such as Bellman-Eluder dimension \citep{jin2021bellman}.

\begin{definition}\label{Definition:kappa}
For any function $Q'\in\cQ_{h+1}$, we define the set $\cQ_h(\epsilon, Q') = \{Q\in\cQ_h: \sup_{x,a} |Q(x,a) - \cT^*_h Q'(x,a)|\leq \epsilon\}$ of functions that have small Bellman error with $Q'$ for all state-action pairs. Using this set, we define $\kappa(\epsilon) = \sup_{Q\in\cQ}\sum_{h=1}^H -\ln p_0^h(\cQ_h(\epsilon, Q_{h+1}))$.
\end{definition}

The set $\cQ_h(\epsilon, Q_{h+1})$ includes the functions that approximately satisfy the Bellman equation and $p_0^h(\cQ_h(\epsilon, Q_{h+1}))$ denotes the probability that is assigned on this set by the prior. From the definition, it is clear that the complexity $\kappa(\epsilon)$ takes a small value if the prior is high for any $Q \in \cQ$ and in that case, it is equivalent to an approximate completeness assumption. Please refer to \citet{dann2021provably} for further details on this metric. 

We denote the sampled posterior by Algorithm~\ref{Algorithm:FG-LMC} at episode $k$ by $q_k'$, which generally deviates from the true posterior $q_k$ defined in \eqref{Eq:posterior_of_Q} due to the inherent limitations of approximate samplers discussed in \Cref{sec:different_samplers}. At each episode k, we define the sampling error $\delta_k = TV(q_k, q_k')$ as the TV distance between the true posterior and the approximate posterior generated by our sampler. 

Using the quantities defined above, we present our first theorem: a frequentist (worst-case) expected regret bound for Algorithm~\ref{Algorithm:FG-LMC}:
\begin{theorem}\label{thm:regret_general}
Under Assumption~\ref{assumption_realizability},~\ref{assumption_boundedness} and~\ref{assumption_completeness}, if $\eta\leq 2/5b^2$, then
\begin{align*}
    & \EE[\text{Regret}(K)] \leq \frac{\lambda}{\eta}\decoup + \frac{2K}{\lambda}\kappa(b/K^2) + \frac{6H}{\lambda} + \frac{b}{K} 
    + \sum_{k=1}^K \left[\left(\frac{\eta}{4\lambda}b^2H(k-1) +  b\right)\cdot \delta_k\right],
\end{align*}
where the expectation incorporates the inherent randomness of the MDP through samples drawn from it and the algorithm's own stochastic elements. If we further set $\eta = 1/4b^2$, $\lambda = \sqrt{K\kappa(b/K^2)/b^2\decoup}$ and assume $\lambda b^2 \geq 1$ and without loss of generality that $b\geq 16$, then the bound becomes
\begin{align}\label{eq:regret_simplified}
   \EE[\text{Regret}(K)] = O\left(b\sqrt{\decoup\kappa(b^2/K)K} + b^2H + \frac bK\right) + \frac{1}{16}b^2\sum_{k=1}^K k\delta_k.
\end{align}
\end{theorem}
\begin{remark}
    It is important to emphasize that the theorem establishes the relationship between the regret and the sampling error, without necessarily asserting that the sampling error $\delta_k$ is small for general function classes. In Section~\ref{subsec:linear_mdp}, we delve deeper into controlling the sampling error with respect to the sampling complexity for linear MDPs.
\end{remark}
\begin{remark}
    The final term in ~\eqref{eq:regret_simplified} highlights that during initial episodes (small $k$), our approximate samplers can have relaxed accuracy requirements. This aligns with the algorithm's exploratory phase, where precise posterior estimates are less crucial compared to later exploitation stages when accurate value estimation becomes critical.
\end{remark}

\begin{remark}
 The derived regret bound from \Cref{thm:regret_general}, can be decomposed into two parts: 
\begin{equation}\label{eq:r_origin_sampler}
R_{origin}  = \frac{\lambda}{\eta} \decoup + \frac{2K}{\lambda}\kappa(b/K^2) 
     + \frac{6H}{\lambda} + \frac{b}{K},
   \quad\mathrm{and}\quad 
R_{sample}  = \sum_{k=1}^K \left[\left(\frac{\eta}{4\lambda}b^2H(k-1) +  b\right)\cdot \delta_k\right].
\end{equation}
Here $R_{sample}$ accounts for the sampling error. It is noteworthy that $R_{origin}$ mirrors Theorem 1 from~\cite{dann2021provably}, and consequently, we adhere to their analytical framework for deriving this part. For $R_{sample}$, we separately examine different samplers for their respective sampling complexities.   
\end{remark}

\begin{remark}
Note that if we can do TS exactly at each step, i.e. $\delta_k=0$ for all $k$, then \Cref{thm:regret_general} reduces to Theorem 1 in~\cite{dann2021provably}. Also as discussed in their work, the decoupling coefficient $\decoup$ can vary in different settings.
\end{remark}

\subsection{Analysis of Errors Induced by Approximating Samplers}\label{subsection:decompose_deltak}

It is important to note that $\delta_k$ within \Cref{thm:regret_general} cannot be directly controlled by the chosen approximating samplers employed in \Cref{Algorithm:FG-LMC}. Therefore, a further decomposition of this term is necessary (see details in \Cref{sec:sampling_error}):
\begin{proposition}\label{prop:delta}
Let $\delta_k^h$ be the sampling error (in the total variation sense) induced by our sampler at step $h\in[H]$ and episode $k\in[K]$ and let $\delta_k$ be as defined in \Cref{sec:regret_bound_for_general_function_class}. Then $\delta_k \leq \sum_{h=1}^H\delta_k^h$.
\end{proposition}
\begin{remark}
\Cref{prop:delta} allows us to decompose the sampling error $\delta_k$ into individual components $\delta_k^h$, representing the total variation distance at step $h$ within episode $k$. Notably, these individual components $\delta_k^h$ are directly controllable by our approximate samplers. This translates to the overall sampling error $\delta_k \leq \sum_{h=1}^H \delta_k^h$ highlighting the crucial role of sampler accuracy in each step in managing the cumulative error $\sum_{h=1}^H \delta_k^h$.
\end{remark}
\begin{remark}
    One should expect both sampling error and truncation error to contribute to the total error $\delta_k$, however, by considering truncation as a transport map between probability distributions and assuming that the exact target distribution is invariant with respect to the truncation due to its belonging to the given function class $\cQ$, we are able to disregard the effect of truncation error using the data-processing inequality. See Appendix~\ref{sec:sampling_error} for details.
\end{remark}
The proposition implies that the regret arising from the approximate sampler defined in~\eqref{eq:r_origin_sampler} is upper-bounded by $ R_{sample}\leq \sum_{k=1}^K \big[(\eta/4\lambda)b^2H(k-1) +  b)\cdot\sum_{h=1}^H\delta_k^h\big]$.

\subsection{Applications to Linear MDPs}\label{subsec:linear_mdp}

\begin{table*}[!htpb]
\begin{center}
\begin{small}
\resizebox{\textwidth}{!}{
\begin{tabular}{l  a  b  a  b  a}
\hline
\rowcolor{LightCyan}&&&Computational & Sampling\\
\rowcolor{LightCyan}\multirow{-2}{*}{Algorithm}& \multirow{-2}{*}{Regret} & \multirow{-2}{*}{Exploration} & Tractability & Complexity\\ \hline
LSVI-UCB \citep{jin2019provably} & $\Tilde{\mathcal{O}}(d^{3/2} H^{3/2} \sqrt{T} )$ &  UCB & Yes & NA \\ \hline
OPT-RLSVI \citep{zanette2019frequentist} & $\Tilde{\mathcal{O}}(d^2 H^2 \sqrt{T} )$ &  TS & Yes & NA\\ \hline

ELEANOR \citep{zanette2020learning} & $\Tilde{\mathcal{O}}(dH^{3/2}\sqrt{T})$ & Optimism & No  & NA\\ \hline
CPS \citep{dann2021provably} & $\Tilde{\mathcal{O}}(dH^{2}\sqrt{T})$ & FGTS & No  & NA\\ \hline
LSVI-PHE \citep{ishfaq2021randomized} & $\Tilde{\mathcal{O}}(d^{3/2} H^{3/2} \sqrt{T})$ & TS & Yes  & NA\\\hline 
LMC-LSVI \citep{ishfaq2023provable} & $\Tilde{\mathcal{O}}(d^{3/2}H^{3/2}\sqrt{T})$ & LMC & Yes & $\Tilde{\Theta}(\frac{\kappa^3K^3H^3}{d\ln(dT)})$\\\hline 
LSVI-ASE with LMC sampler & $\Tilde{\mathcal{O}}(dH^{3/2}\sqrt{T})$ & FGTS \& LMC & Yes & $\Tilde{\Theta}(\frac{\kappa^3K^3H^3}{d\ln(dT)})$\\\hline
LSVI-ASE with ULMC sampler & $\Tilde{\mathcal{O}}(dH^{3/2}\sqrt{T})$ & FGTS \& ULMC & Yes & $\Tilde{\Theta}(\frac{\kappa^{3/2}K^2H^2}{\sqrt{d\ln(dT)}})$ \\\hline
\end{tabular}
}
\end{small}
\end{center}
\caption{Regret upper bound for episodic, non-stationary, linear
MDPs. Here, computational tractability refers to the ability of a computational problem to be solved in a reasonable amount of time using a feasible amount of computational resources.
\label{table:bounds}}
\end{table*}

A concrete example where we can interpret the regret bound from \Cref{thm:regret_general} is the linear MDP \citep{jin2019provably, yang2020reinforcement, yang2019sample} setting.
\begin{definition}\label{def:linear_MDP}
(Linear MDP). An MDP $(\mathcal{S}, \mathcal{A}, H, \mathbb{P}, r)$ is said to be a linear MDP with a feature $\phi:\mathcal{S}\times \mathcal{A}\rightarrow \mathbb{R}^d$, if for any $h\in [H]$,  there exist $d$ unknown (signed) measures $\mu_h = (\mu_h^{(1)}, \ldots ,\mu_h^{(d)})$ over $\mathcal{S}$ and an unknown vector $\theta_h\in\mathbb R^d$ such that for any $(x,a)\in \mathcal{S}\times \mathcal{A}$, we have $\mathbb{P}_h(\cdot|x, a) = \langle\phi(x,a), \mu_h(\cdot)\rangle$ and $r_h(x,a) = \langle\phi(x,a), \theta_h\rangle$.

Without loss of generality, we assume $\|\phi(x, a)\|_2 \leq 1$ for all $(x, a) \in \mathcal{S}\times \mathcal{A}$, and
$\max\{\|\mu_h(\mathcal{S})\|_2, \|\theta_h\|_2\} \leq\sqrt{d}$ for all $h \in [H]$. 
\end{definition}

We first bound $\kappa(\epsilon)$ defined in \Cref{Definition:kappa} for linear MDP. While previous work by~\cite{dann2021provably} provides bounds with a uniform prior distribution over the function class, it does not align with the way TS algorithms are implemented in practice. For this, we consider a Gaussian distribution as the prior distribution.
\begin{lemma}\label{lemma:kappa_linearMDP}
If the stage-wise priors $p_0^h$ are chosen as $\mathcal{N}(0, \sqrt{d}HI_d)$, then $\kappa(\epsilon) = dHO(\ln(dH/\epsilon))$.
\end{lemma}
\begin{remark}
While Gaussian priors are commonly used \citep{he2015delving, goodfellow2016deep}, we highlight that the prior distribution $p_0^h$ can be any distribution in practice, as long as a suitable bound for $\kappa(\epsilon)$ exists. This flexibility allows for incorporating domain-specific knowledge into the prior.
\end{remark}

We can now illustrate Theorem~\ref{thm:regret_general} for linear MDP:
\begin{corollary}\label{Corollary:linear_MDP_regret}
If we set $\eta = 2/5H^2$ and $\lambda = \sqrt{K\kappa(H/K^2)/dH^3(1+\ln(2T))}$, then the expected regret of \Cref{Algorithm:FG-LMC} after $K$ episodes in a linear MDP is bounded as
\begin{align*}
\textstyle\EE[\text{Regret}(K)] = O(dH^\frac32\sqrt{T}\ln(dT)) + \sum_{k=1}^K \alpha_k \big(\sum_{h=1}^H \delta_k^h\big),
\end{align*}
where $\alpha_k = O(\sqrt{\ln(dT)/K}H^2k)$ and $T = HK$ is the total number of steps.
\end{corollary}

\subsection{Sampling Complexity of Different Samplers}
In this section, we characterize the sampling complexity of the proposed algorithms to demonstrate that we can achieve the desired regret bound as long as the chosen sampler is executed a sufficient number of times. We begin by establishing an appropriate notion of complexity.
\begin{definition}
(Sampling Complexity) The agent has access to the gradient $\nabla_w Q(w;\phi(x, a))$ for
any $w\in\mathbb{R}^d$. Then, if $\nabla_w Q$ is evaluated $G_k$ times at any episode $k\in[K]$, then
we define $G_k$ as the sampling complexity at episode $k$,
and $SC = \sum_{k\in[K]}G_k$ be the cumulative sampling complexity.    
\end{definition}
\begin{remark} 
In \Cref{Algorithm:FG-LMC}, $G_k$ specifically represents the total number of iterations employed by our approximate samplers from line \ref{line:begin_inner_for_loop} to line \ref{line:end_inner_for_loop} during episode $k$. It follows that within our analysis, $G_k = J_k$ and $SC =\sum_{k\in[K]} G_k = \sum_{k\in[K]} J_k$.
\end{remark}

\begin{theorem}\label{Theorem:sample_complexity_bound}
Consider a linear MDP defined in Definition~\ref{def:linear_MDP}. Assume that there exists $\kappa>0$ such that for any $ (k, h) \in [K] \times [H]$, the loss function defined in~\eqref{eq:loss_h>0} satisfies $M_{k, h} I \ge \nabla^2 L_h^k \ge m_{k, h}I$ and $M_{k, h}/m_{k, h}\leq \kappa$ for some $M_{k, h} \geq m_{k, h}>0$.
Then we can achieve the regret bound of $O(dH^\frac32\sqrt{T}\ln(dT))$ using our approximate samplers with the cumulative sampling complexity stated below:\newline
(1) LMC: $SC = \Tilde{\Theta}(\frac{\kappa^3K^3H^3}{d\ln(dT)})$ with step size $\tau_{k, h} = \Tilde{\Theta}(\frac{d\ln(dT)}{M_{k, h} H^2k^2\kappa})$;\newline
(2) ULMC: $SC = \Tilde{\Theta}(\frac{\kappa^{3/2}K^2H^2}{\sqrt{d\ln(dT)}})$ with step size $\tau_{k, h} = \Tilde{\Theta}(\frac{\sqrt{d\ln(dT)}}{M_{k, h} Hk})$.
\end{theorem}
\begin{remark}
\Cref{Theorem:sample_complexity_bound} reveals a critical relationship between the choice of sampling method and the sampling complexity of \Cref{Algorithm:FG-LMC}. Leveraging established results on demonstrating the faster mixing of ULMC over LMC in strongly log-concave settings (see \Cref{section:samplers} for details), the theorem confirms that \Cref{Algorithm:FG-LMC}, when employing ULMC, achieves the desired accuracy with lower data requirements than its LMC-based counterpart. This aligns with the intuitive notion that ULMC's momentum-based exploration enables faster learning, thereby reducing the necessary data for effective Thompson sampling.
\end{remark}

\section{Experiments}\label{sec:experiments}

In this section, we provide an empirical evaluation of our proposed algorithms with deep Q-networks (DQNs) \citep{mnih2015human} in two sets of environments: (1) the N-chain environment \citep{osband2016deep} and (2) the Atari game suite \citep{bellemare2013arcade, taiga2019bonus}. We evaluate our algorithms with different implementations. In particular, we implement the \Cref{Algorithm:FG-LMC} with different choices of prior terms and samplers. By choosing Feel-Good exploration prior term in \eqref{eq:loss_h>0} and underdamped Langevin Monte Carlo sampler with adaptive bias term in \eqref{Eq:aulmc}, we implement the \Cref{Algorithm:FG-LMC} named as Feel-Good Underdamped Langevin Monte Carlo Deep Q-Network (FG-ULMCDQN). We implement the \Cref{Algorithm:FG-LMC} with the Feel-Good exploration prior term and the adaptive Langevin Monte Carlo sampler introduced in \citet{ishfaq2023provable}, named Feel-Good Langevin Monte Carlo Deep Q-Network (FG-LMCDQN). We also provide an implementation for the \Cref{Algorithm:FG-LMC} without the Feel-Good exploration prior term named Underdamped Langevin Monte Carlo Deep Q-Network (ULMCDQN). Then we evaluate our implementations in the above mentioned environments. Our code is available at \url{https://github.com/panxulab/LSVI-ASE}.

\subsection{Experiments in N-Chain}
We demonstrate that our proposed algorithms can explore effectively in sparse-reward environment by conducting experiments in $N$-Chain environment \citep{osband2016deep} that demands deep exploration capabilities to perform well. An $N$-chain environment can be constructed by a chain of $N > 3$ states denoted by $s_1, s_2, \ldots, s_N$. Each episode of interaction, which starts at state $s_2$, lasts for $N + 9$ steps and in each step the agent can either move to the left or right. A myopic agent would gravitate toward state $s_1$ which has a small reward of $r = 0.001$ whereas an efficient agent with deep exploration capabilities would try to reach state $s_N$ which has a larger reward of $r = 1$. As each episode runs for $N + 9$ steps, the optimal return for an episode is $10$. We refer the reader to \Cref{sec:nchain} for a depiction of the environment.

\begin{wrapfigure}[16]{r}{0.5\textwidth}
    \vspace{-0.4in}  
  \begin{center}
    \includegraphics[width=0.50\textwidth]{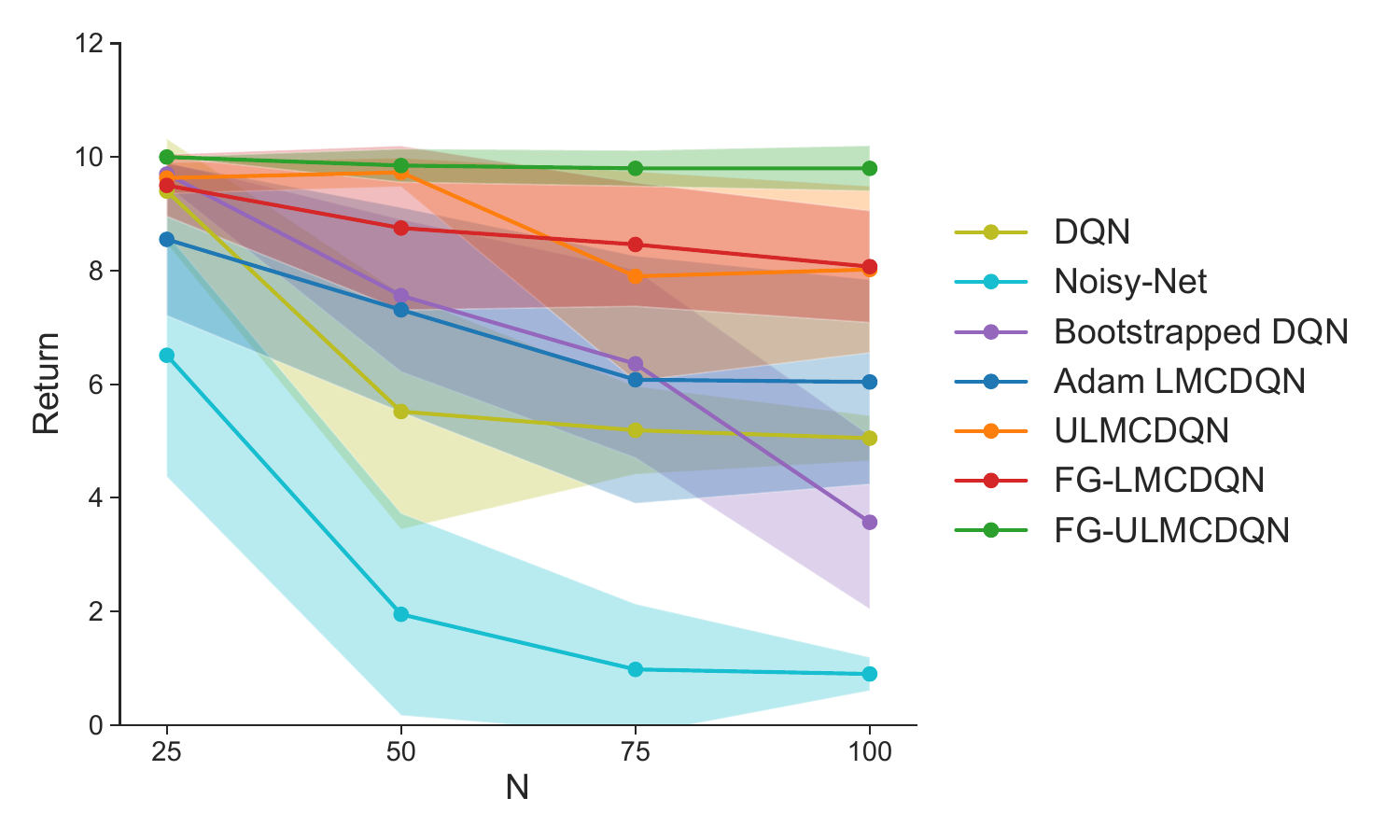}
  \end{center}
  \vspace{-0.2in}
  \caption{A comparison of different methods in $N$-chain with different chain lengths $N$. As $N$ increases, the exploration hardness increases. All results are averaged over $20$ runs and the shaded areas represent $95\%$ confidence interval.}
  \label{fig:nchain_best}
\end{wrapfigure}

In our experiments, we progressively increase the difficulty level by setting $N$ to be $25$, $50$, $75$, and $100$. For each chain length, we run each learning algorithm for $10^5$ steps across $20$ seeds.  As baseline algorithms, we use 
DQN \citep{mnih2015human}, Bootstrapped DQN \citep{osband2016deep}, Noisy-Net \citep{fortunato2017noisy} and Adam LMCDQN \citep{ishfaq2023provable}. The performance of each algorithm in each run is measured by the mean return of the last $10$ evaluation episodes. We sweep the learning rate and pick the one with the best performance for each algorithm. For our algorithms which use ULMC as a sampler, we sweep the friction coefficient $\gamma$. For FG-LMCDQN and FG-ULMCDQN, we sweep the weight for the feel-good prior term $\eta$ in the loss function. Please check \Cref{sec:nchain} for further details.

\Cref{fig:nchain_best} shows the performance of our proposed algorithms as well as the baseline algorithms under different chain lengths. The solid lines represent the average return over $20$ random seeds and the shaded areas represent the $95\%$ confidence interval. For all of our proposed algorithms, namely ULMCDQN, FG-ULMCDQN and FG-LMCDQN, we set $J_k = 4$ in \Cref{Algorithm:FG-LMC} for all chain lengths. 

From \Cref{fig:nchain_best}, we see that as the chain length $N$ increases, the performance of the baselines drops drastically. Whereas, our FGTS based algorithms FG-LMCDQN and FG-ULMCDQN are able to maintain steady performance. In particular, we would like to highlight that FG-ULMCDQN is able to get almost close to the optimal return of $10$ for all values of $N$, showing the benefit of using Feel Good prior along with underdamped LMC together in environments where deep exploration is absolutely necessary to perform well.

\subsection{Experiments in Atari Games}

\begin{figure*}
    \centering
    \includegraphics[width=\textwidth]{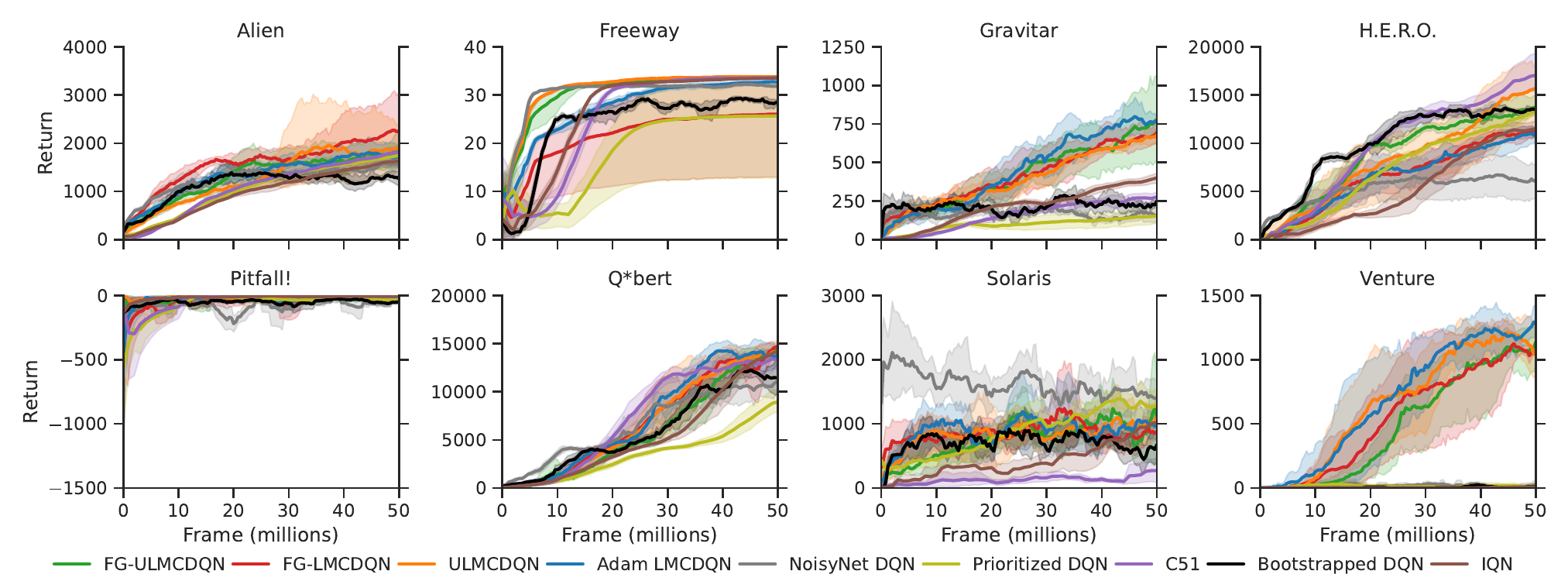}
    \caption{The return curves of various algorithms in eight hard Atari tasks over 50 million training frames. Solid lines correspond to the median performance over 5 random seeds, and the shaded areas correspond to $95\%$ confidence interval.}
    \label{fig:atari}
\end{figure*}

We evaluate our algorithms in 8 visually complicated hard exploration games, namely Alien, Freeway, Gravitar, H.E.R.O., Pitfall, Qbert, Solaris, and Venture from the Atari game suite \citep{bellemare2013arcade, taiga2019bonus}. As classified in \citet{taiga2019bonus}, among these games, Alien, H.E.R.O., and Qbert are dense reward environments and Freeway, Gravitar, Pitfall, Solaris, and Venture are sparse reward environments. In our experiments, we set $J_k = 1$ in \Cref{Algorithm:FG-LMC} to finish the training in a reasonable time. Following \citep{ishfaq2023provable}, we also incorporate the double Q trick \citep{van2010double, van2016deep} in our implementation. As baselines we consider  Adam LMCDQN \citep{ishfaq2023provable}, Noisy-Net \citep{fortunato2017noisy}, Prioritized DQN \citep{schaul2015prioritized}, C51 \citep{bellemare2017distributional}, Bootstrapped DQN \citep{osband2016deep} and  IQN \citep{dabney2018implicit}. All algorithms are trained for 50M frames (i.e., 12.5M steps) and run for 5 different random seeds. We refer the reader to \Cref{sec:atari} for further details on training and hyper-parameter choices. \Cref{fig:atari} depicts the learning curves of all algorithms in 8 Atari games. Compared to the baseline algorithms, our algorithms appear to be quite competitive despite being much simpler in implementation. We highlight the advantages of approximate sampling based algorithms in Gravitar and Venture. 
\begin{figure}[h]
\vspace{-0.3cm}
  \begin{subfigure}[b]{0.50\textwidth}
    \includegraphics[width=\linewidth]{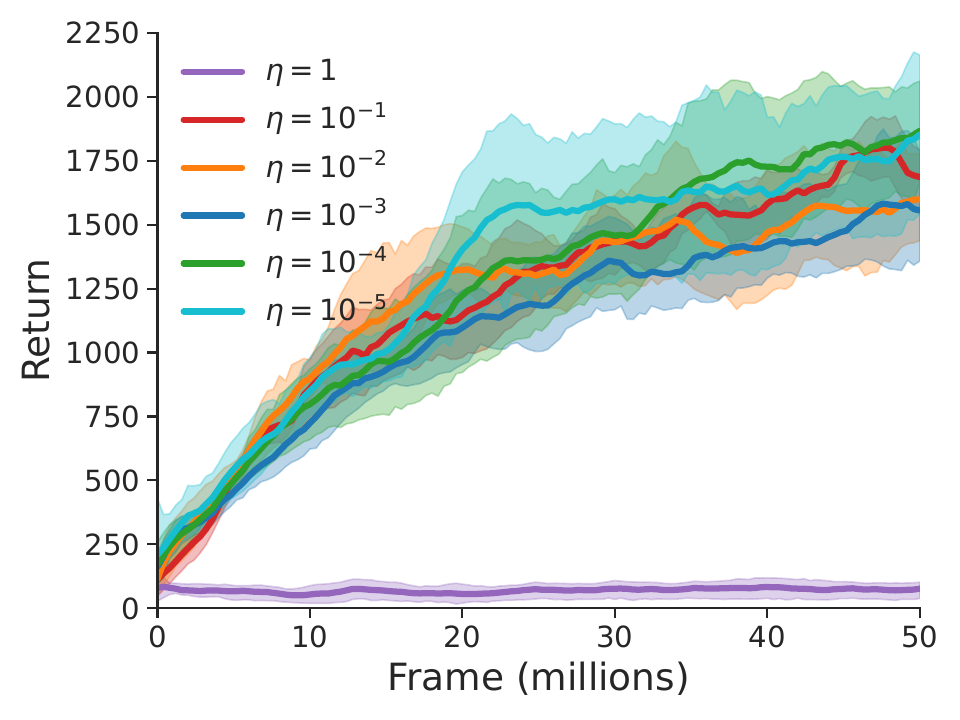}
    \caption{Different FG prior weight $\eta$ in FG-ULMCDQN}\label{fig:alien_eta}
  \end{subfigure}%
  \hspace*{\fill}   %
  \begin{subfigure}[b]{0.50\textwidth}
    \includegraphics[width=\linewidth]{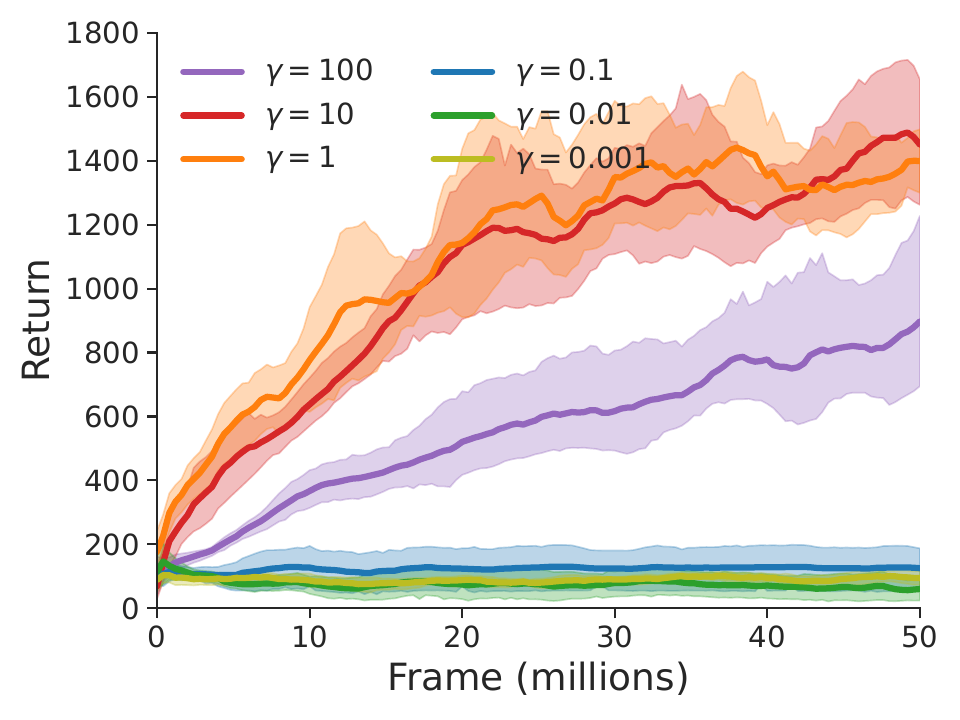}
    \caption{Different friction coefficient $\gamma$ in ULMCDQN}\label{fig:alien_gamma}
  \end{subfigure}
\caption{(a) A comparison of FG-ULMCDQN with different values of weight $\eta$ for the feel good prior term in Alien. Solid lines correspond to the average performance over 5 random seeds, and shaded areas correspond to $95\%$ confidence interval. The performance of FG-ULMCDQN is not very sensitive to the values of $\eta$ as long it is not very large. (b) A comparison of ULMCDQN with different values of the friction coefficient  $\gamma$ in Alien.} \label{fig:comp}
\end{figure}

\textbf{Sensitivity Analysis.}
In \Cref{fig:alien_eta}, we draw the learning curves of FG-ULMCDQN with different weight factor $\eta$ for the FG prior term. We observe that as long the value of $\eta$ is not too high, the performance of FG-ULMCDQN is less sensitive to the value of $\eta$. In \Cref{fig:alien_gamma}, we observe that for very high or low value of friction coefficient $\gamma$, the performance of ULMCDQN collapses. 

\section{Conclusion}
This work introduces a novel algorithmic framework that leverages efficient approximate samplers to make FGTS practical for real-world RL. Unlike prior approaches reliant on unrealistic sampling oracles, our framework enables computationally feasible exploration. Furthermore, our theoretical analysis provides a deeper understanding of the relationship between samplers and regret in FGTS algorithms. This newfound knowledge paves the way for practical exploration strategies with strong provable guarantees. Notably, our algorithm achieves an improved regret bound in linear MDPs, and showcases consistent performance in deep exploration environments.

Future directions include exploring the integration of alternative approximate samplers within our framework. Promising candidates include Metropolis-adjusted Langevin Acceptance (MALA)~\citep{besag1995bayesian} and various proximal sampling algorithms~\citep{lee2021structured}. Investigating efficient methods to incorporate these samplers into the RL setting while maintaining the framework's strengths will further enhance its applicability to diverse exploration problems.

\subsubsection*{Acknowledgments}
We gratefully acknowledge funding from the Canada CIFAR AI Chairs program, the Reinforcement Learning and Artificial Intelligence (RLAI) laboratory, Mila - Quebec Artificial Intelligence Institute, the Natural Sciences and Engineering Research Council (NSERC) of Canada, the US National Science Foundation (DMS-2323112) and the Whitehead Scholars Program at the Duke University School of Medicine. 
Qingfeng Lan would also like to acknowledge Alberta Innovates for the support they provided him for his research.

\bibliography{references}
\bibliographystyle{rlc}

\newpage
\appendix

\section{Regret Analysis}\label{sec:proof_regret_general}

\subsection{Proof of Main Results}
In this section, we restate and provide the proof of our main result \Cref{thm:regret_general} and its corollary for the linear MDP.

\begin{theorem}\label{thm:regret_general_restate}[Restatement of \Cref{thm:regret_general}]
Under Assumption~\ref{assumption_realizability},~\ref{assumption_boundedness} and~\ref{assumption_completeness}, if $\eta\leq 2/5b^2$, then
\begin{align*}
    & \E[\text{Regret}(K)] \leq \frac{\lambda}{\eta}\decoup + \frac{2K}{\lambda}\kappa(b/K^2) + \frac{6H}{\lambda} + \frac{b}{K} 
    + \sum_{k=1}^K \left[\left(\frac{\eta}{4\lambda}b^2H(k-1) +  b\right)\cdot \delta_k\right],
\end{align*}
where the expectation incorporates the inherent randomness of the MDP environment through samples drawn from it and the algorithm's own stochastic elements. If we further set $\eta = 1/4b^2$ and $\lambda = \sqrt{\frac{K\kappa(b/K^2)}{b^2\decoup}}$ and assume $\lambda b^2 \geq 1$ and without loss of generality that $b\geq 16$, then the bound becomes
\begin{align}\label{eq:regret_simplified_restate}
   \E[\text{Regret}(K)] = O\left(b\sqrt{\decoup\kappa(b^2/K)K} + b^2H + \frac bK\right) + \frac{1}{16}b^2\sum_{k=1}^K k\delta_k.
\end{align}
\end{theorem}

\begin{proof}[Proof of \Cref{thm:regret_general}]

Given any policy $\pi^k$ and initial state $x_1^k$, by Lemma~\ref{lemma:regret-decomp}, we can decompose the regret:
\begin{align*}
    \text{REG}_k & = \sum_{h=1}^H \EE_{\pi^k} \left[Q_h^k(x_h^k,a_h^k) - r_h(x_h^k,a_h^k) - \EE_{x_{h+1} \sim \PP_h(\cdot | x_h^k, a_h^k)}\max_{a\in\cA}Q_{h+1}^k(x_{h+1},a)\right] \\
    & \qquad - \left[V_1^k(x_1^k) - V_1^*(x_1^k)\right]\\
    & =: \sum_{h=1}^H \text{BE}_k^h - \text{FG}_k.    
\end{align*}
and hence we can rewrite the expected regret of episode $k$ (scaled by $\lambda$) as
\begin{align*}
   \lambda\EE_{Q^k\sim q_k'} \text{REG}_k & =\lambda \EE_{Q^k\sim q_k'}\bigg[\sum_{h=1}^H \text{BE}_k^h - \text{FG}_k\bigg]\\
   & = \EE_{Q^k\sim q_k'}\sum_{h=1}^H\bigg[\lambda \text{BE}_k^h -\frac\eta4 \sum_{s=1}^{k-1}\EE_{[x_h, a_h]\sim p(\cdot| Q^s, x_1)}[\cE_h(Q^k;x_h,a_h)]^2\bigg]\\
    &\quad + \EE_{Q^k\sim q_k'}\bigg[\sum_{h=1}^H\frac\eta4\sum_{s=1}^{k-1}\EE_{[x_h, a_h]\sim p(\cdot| Q^s, x_1)}[\cE_h(Q^k;x_h,a_h)]^2 - \lambda\text{FG}_k\bigg]\\
    & =: F_k^{dc} + F_k^{\kappa}.
\end{align*}
Summing over $k=1,2,\ldots,K$, we obtain that
\begin{align*}
    \lambda\E [\text{Regret}(K)] =  \sum_{k=1}^K F_k^{dc} + \sum_{k=1}^K F_k^{\kappa}.
\end{align*}
Using Definition~\ref{def:dc}, we can bound the first term by
\begin{align}\label{eq:decoup}
    \sum_{k=1}^K F_k^{dc} \leq \frac{\lambda^2}{\eta}\decoup 
\end{align}
for any $\frac{\eta}{4\lambda}<1$. For the $F_k^\kappa$ term, by Assumption~\ref{assumption_boundedness}, $\cE_h(Q^k;x_h,a_h)^2\leq b^2$ and $|\text{FG}_k|\leq b$. Therefore,
\begin{align*}
    \left \lvert\sum_{h=1}^H\frac\eta4\sum_{s=1}^{k-1}\EE_{[x_h, a_h]\sim p(\cdot| Q^s, x_1)}[\cE_h(Q^k;x_h,a_h)]^2 - \lambda\text{FG}_k\right\rvert \leq \frac{\eta}{4}b^2H(k-1) + \lambda b.
\end{align*}
Then by the property of total variation distance,
\begin{align*}
    F_k^{\kappa} \leq \EE_{Q^k\sim q_k}\bigg[\sum_{h=1}^H\frac\eta4\sum_{s=1}^{k-1}\EE_{[x_h, a_h]\sim p(\cdot| Q^s, x_1)}[\cE_h(Q^k;x_h,a_h)]^2 - \lambda\text{FG}_k\bigg] + \bigg[\frac{\eta}{4}b^2H(k-1) + \lambda b\bigg]\cdot\delta_k,
\end{align*}
where $\delta_k = TV(q_k, q_k').$ The first term on the right handside is exactly the analog of $F_k^\kappa$ with the expectation taken over the exact target distribution of our algorithm. By Theorem~\ref{thm:regret} (which is Theorem 1 of~\cite{dann2021provably}),
\begin{align*}
   \sum_{k=1}^K\EE_{Q^k\sim q_k}\bigg[\sum_{h=1}^H\frac\eta4\sum_{s=1}^{k-1}\EE_{[x_h, a_h]\sim p(\cdot| Q^s, x_1)}[\cE_h(Q^k;x_h,a_h)]^2 - \lambda\text{FG}_k\bigg] \leq 2K\kappa(b/K^2) + 6H + \frac{\lambda b}{K}. 
\end{align*}
Hence
\begin{align*}
    \sum_{k=1}^K F_k^\kappa \leq 2K\kappa(b/T^2) + 6H + \frac{\lambda b}{K} + \sum_{k=1}^K \bigg[\frac{\eta}{4}b^2H(k-1) + \lambda b\bigg]\cdot\delta_k.
\end{align*}
Combining this inequality with~\eqref{eq:decoup}, we obtain that
\begin{align*}
    \E [\text{Regret}(K)] & \leq \frac{1}{\lambda}\bigg[\sum_{k=1}^K F_k^{dc} + \sum_{k=1}^K F_k^{\kappa}\bigg]\\ & \leq \frac{\lambda}{\eta}\decoup + \frac{2K}{\lambda}\kappa(b/K^2) + \frac{6H}{\lambda} + \frac{b}{K} 
    + \sum_{k=1}^K \bigg[\frac{\eta}{4\lambda}b^2H(k-1) +  b\bigg]\cdot\delta_k.
\end{align*}
If we further set $\eta = 1/4b^2$ and $\lambda = \sqrt{\frac{K\kappa(b/K^2)}{b^2\decoup}}$, then a direct calculation gives us
\begin{align*}
    \frac{\lambda}{\eta}\decoup + \frac{2K}{\lambda}\kappa(b/K^2) + \frac{6H}{\lambda} = 6b\sqrt{\decoup\kappa(b/K^2)K} + 6H\sqrt{\frac{b^2\decoup}{\kappa(b/K^2)K}}.
\end{align*}
Since $\lambda b^2\geq 1$, we have
\begin{align*}
    \sqrt{\frac{b^2\decoup}{\kappa(b/K^2)K}} \leq b^2.
\end{align*}
And since $b\geq 16$, we have
\begin{align*}
    \frac{\eta}{4\lambda}b^2 H(k-1) + b \leq \frac{1}{16}b^2(k-1) + b \leq \frac{1}{16}b^2k.
\end{align*}
Combining these results, we obtain~\eqref{eq:regret_simplified}.
\end{proof}
Next, we restate \Cref{Corollary:linear_MDP_regret} and provide proof for it.

\begin{corollary}\label{Corollary:linear_MDP_regret_restate}
Assume \Cref{Algorithm:FG-LMC} is run on a $d$-dimensional linear MDP. If we set $\eta = \frac{2}{5H^2}$ and $\lambda = \sqrt{\frac{K\kappa(H/K^2)}{dH^3(1+\ln(2T))}}$, then the expected regret after $K$ episodes is bounded as
\begin{align*}
    \E[\text{Regret}(K)] \leq O(dH^\frac32\sqrt{T}\ln(dT)) + \sum_{k=1}^K \alpha_k \big(\sum_{h=1}^H \delta_k^h\big),
\end{align*}
where $\alpha_k = O\bigg(\sqrt{\frac{\ln(dT)}{K}}H^2k\bigg)$ and $T = HK$ is the total number of steps. 
\end{corollary}

\begin{proof}[Proof of \Cref{Corollary:linear_MDP_regret}]
For linear MDPs, by proposition~\ref{proposition:linear_MDP}, $b=H$, $\decoup \leq 2dH(1+2\ln(KH))$ and $\kappa(H/K^2) = dH\ln(dHK)$. Therefore $R_{origin}$ defined in~\eqref{eq:r_origin_sampler} becomes $O(dH^\frac32\sqrt{T}\ln(dT))$. Moreover, by a direct calculation,
\begin{align*}
    \alpha_k := \frac{\eta}{4\lambda}b^2H(k-1) + b = O\bigg(\sqrt{\frac{\ln(dT)}{K}}H^2k\bigg)
\end{align*}
and hence $R_{sample}$ defined in~\eqref{eq:r_origin_sampler} becomes
\begin{align*}
    \sum_{k=1}^K \alpha_k \delta_k \leq \sum_{k=1}^K \alpha_k \big(\sum_{h=1}^H \delta_k^h\big),
\end{align*}
where the last inequality is due to Proposition~\ref{prop:delta}. In conclusion,
\begin{align*}
    \E[\text{Regret}(K)] \leq R_{origin} + R_{sample} = O(dH^\frac32\sqrt{T}\ln(dT)) + \sum_{k=1}^K \alpha_k \big(\sum_{h=1}^H \delta_k^h\big).
\end{align*}
\end{proof}
With this result, we are ready to prove \Cref{Theorem:sample_complexity_bound}:

\begin{theorem}[Restatement of \Cref{Theorem:sample_complexity_bound}]\label{Theorem:sample_complexity_bound_restate}
Consider a linear MDP defined in Definition~\ref{def:linear_MDP}. Assume that there exists $\kappa>0$ such that for any $ (k, h) \in [K] \times [H]$, the loss function defined in~\eqref{eq:loss_h>0} satisfies for some $M_{k, h} \geq m_{k, h}>0$:
\begin{align*}
    M_{k, h} I \ge \nabla^2 L_h^k \ge m_{k, h}I,\quad M_{k, h}/m_{k, h}\leq \kappa.
\end{align*}
Then we can achieve the regret bound $O(dH^\frac32\sqrt{T}\ln(dT))$ using our approximate samplers with the cumulative sampling complexity stated below:\newline
(1) LMC: $SC = \Tilde{\Theta}(\frac{\kappa^3K^3H^3}{d\ln(dT)})$ with step size $\tau_{k, h} = \Tilde{\Theta}(\frac{d\ln(dT)}{M_{k, h} H^2k^2\kappa})$;\newline
(2) ULMC: $SC = \Tilde{\Theta}(\frac{\kappa^{3/2}K^2H^2}{\sqrt{d\ln(dT)}})$ with step size $\tau_{k, h} = \Tilde{\Theta}(\frac{\sqrt{d\ln(dT)}}{M_{k, h} Hk})$.
\end{theorem}

\begin{proof}[Proof of \Cref{Theorem:sample_complexity_bound}]
We give the proof for ULMC. The proof for LMC is essentially the same using Theorem~\ref{thm:convergence_LMC}. Note that for linear MDP, if we let
\begin{align}\label{eq:sampling_requirement}
    \max_{h\in[H]}\delta_k^h \leq\frac{O(Hd \ln(dT)/\sqrt{K})}{\alpha_k} = O\left(\frac{d\sqrt{\ln(dT)}}{Hk}\right)
\end{align}
at episode $k$, then $\alpha_k\sum_{h=1}^H \delta_k^h \leq O(H^2d\ln(dT)/\sqrt{K})$. This implies that for linear MDP, $\EE[\text{Regret}(K)] = O(dH^\frac32\sqrt{T}\ln(dT))$. For ULMC, by Theorem~\ref{thm:convergence_ULMC}, the requirements in~\eqref{eq:sampling_requirement} can be achieved by setting the step size $\tau_{k, h} = \Tilde{\Theta}(\frac{\sqrt{d\ln(dKH)}}{M_{k, h} Hk})$ and after $N_{k, h} = \Tilde{\Theta}(\frac{\kappa^{3/2}Hk}{\sqrt{d\ln(dKH)}})$ iterations. Summing $N_{k,h}$ over all $k\in[K]$ and $h\in[H]$, we obtain that the cumulative sample complexity is $\Tilde{\Theta}(\frac{\kappa^{3/2}K^2H^2}{\sqrt{d\ln(dKH)}})$.
\end{proof}

\subsection{Useful Lemmas}
In this section, we give some lemmas that are useful in the proof of our main results.
\begin{lemma}[Regret Decomposition]\label{lemma:regret-decomp}
The regret at episode $k$ can be decomposed into two terms
\begin{align*}
    \text{REG}_k = \E_{\pi^k, \PP}\Bigg[\sum_{h=1}^H Q_h^k(x_h^k, a_h^k)) - [\cT_h^* Q^k_{h+1}](x_h^k, a_h^k)\Bigg] - [V_1^k(x_1^k) - V_1^*(x_1^k) ]
\end{align*}
\end{lemma}
\begin{proof}[Proof of \Cref{lemma:regret-decomp}]
Recall that the Bellman optimality operator $\cT_h^*$ maps any state-action function $Q^k_{h+1}$ to 
\begin{equation*}
    [\cT_h^* Q_{h+1}^k](x,a) = r_h(x,a) + \EE_{x'\sim \PP_h(x,a)}[\max_{a'\in \cA}Q^k_{h+1}(x',a')],
\end{equation*}
and hence 
\begin{align*}
    r_h(x_h^k, a_h^k) = [\cT_h^* Q_{h+1}^k](x_h^k, a_h^k) - \EE_{x_{h+1}^k\sim \PP_h(x_h^k, a_h^k)}[\max_{a\in \cA}Q^k_{h+1}(x_{h+1}^k,a)].
\end{align*}
By definition,
\begin{align*}
    V_1^{\pi^k}(x_1^k) & = \mathbb{E}_{\pi^k} \Bigg[\sum_{h=1}^H r_{h}(x_{h}^k,a_{h}^k) \bigg| x_1 = x_1^k\Bigg].
\end{align*}
And then, we have
\begin{align*}
    V_1^{\pi^k}(x_1^k) & = \mathbb{E}_{\pi^k} \Bigg[\sum_{h=1}^H r_{h}(x_{h}^k,a_{h}^k) \bigg| x_1 = x_1^k\Bigg]\\
    & = \mathbb{E}_{\pi^k} \Bigg[\sum_{h=1}^H \Big[[\cT_h^* Q_{h+1}^k](x_h^k, a_h^k) - \EE_{x_{h+1}^k\sim \PP_h(x_h^k, a_h^k)}[\max_{a\in \cA}Q^k_{h+1}(x_{h+1}^k,a)]\Big] \bigg| x_1 = x_1^k\Bigg]\\
    & = \mathbb{E}_{\pi^k, \PP} \Bigg[\sum_{h=1}^H \Big[[\cT_h^* Q_{h+1}^k](x_h^k, a_h^k) - \max_{a\in \cA}Q^k_{h+1}(x_{h+1}^k,a)\Big] \bigg| x_1 = x_1^k\Bigg]\\
    & = \mathbb{E}_{\pi^k, \PP} \Bigg[\sum_{h=1}^H \Big[[\cT_h^* Q_{h+1}^k](x_h^k, a_h^k) - Q^k_{h+1}(x_{h+1}^k,a_{h+1}^k)\Big] \bigg| x_1 = x_1^k\Bigg]\\
    & = \mathbb{E}_{\pi^k, \PP} \Bigg[\sum_{h=1}^H \Big[[\cT_h^* Q_{h+1}^k](x_h^k, a_h^k) - Q^k_{h}(x_h^k,a_h^k)\Big] \bigg| x_1 = x_1^k\Bigg] + \mathbb{E}_{\pi^k}\big[Q_1^k(x_1^k, a_1^k)\big]\\
    & = \mathbb{E}_{\pi^k, \PP} \Bigg[\sum_{h=1}^H \Big[[\cT_h^* Q_{h+1}^k](x_h^k, a_h^k) - Q^k_{h}(x_h^k,a_h^k)\Big] \Bigg] + V_1^k(x_1^k).
\end{align*}
\end{proof}

We restate Theorem 1 of~\cite{dann2021provably} below.

\begin{theorem}[Theorem 1 of \citet{dann2021provably}]\label{thm:regret}
Assume that parameter $\eta\leq \frac{2}{5b^2}$ is set sufficiently small and that Assumption~\ref{assumption_boundedness} holds. Then the expected regret (with Thompson sampling excuted exactly) after $K$ episodes on any MDP $M$ is bounded as
\begin{align*}
    \E[\text{Regret}(K)] \leq \frac{\lambda}{\eta} \decoup + \frac{2K}{\lambda}\kappa(b/T^2) 
     + \frac{6H}{\lambda} + \frac{b}{K},
\end{align*}
where the expectation is over the samples drawn from the MDP and the algorithm's internal randomness.
\end{theorem}

\subsection{Analysis of Sampling Error}\label{sec:sampling_error}
Recalling the procedure for obtaining the $Q$ function at episode $k$: we first sample $w_H$ to obtain $Q_H^k$, and then for $h=H-1, \ldots, 1$, we sample $w_h$ conditional on $Q_{h+1}^k$ to obtain $Q_h^k$. Let's denote the target conditional distribution of $Q_h^k$ given $Q_{h+1}^k$ as $q^k_{h, h+1}$. If Thompson sampling is executed precisely at each step, we should acquire the $Q$ function $Q_h^k \sim q_k^h$ such that $q_k^h = q^k_{h, h+1}(\cdot | Q_{h+1}^k)$
for all $h \in [H]$. We denote the joint distribution of $\{q_k^h\}_{h \in [H]}$ as $q_k$. However, due to the high computational complexity, we can only obtain a sequence of $\{\widetilde{Q}_h^k \sim \widetilde{q}_k^h\}_{h \in [H]}$ that satisfies $\widetilde{q}_k^h = \widetilde{q}^k_{h, h+1}(\cdot | \widetilde{Q}_{h+1}^k),$
where $\widetilde{q}^k_{h, h+1}$ is a transition kernel close to $q^k_{h, h+1}$ and depends on the approximating sampler. We denote the joint distribution of $\{\widetilde{q}_k^h\}_{h \in [H]}$ as $q_k'$. Note that given $\widetilde{Q}_{h+1}^k$ generated by our algorithm, our goal is to sample from $q^k_{h, h+1}(\cdot | \widetilde{Q}_{h+1}^k)$ while we can only obtain a sample close to that, denoted as $\widetilde{q}_k^h$. Our samplers can only control the distance (error) between $\widetilde{q}^k_{h, h+1}(\cdot | \widetilde{Q}_{h+1}^k)$ and $q^k_{h, h+1}(\cdot | \widetilde{Q}_{h+1}^k)$. We denote their total variation distance by $\delta_k^h$. However, the sampling error in Theorem~\ref{thm:regret_general} is in terms of the distance between the joint distributions at episode $k$: $\delta_k = \text{TV}(q_k, q_k')$. Therefore, for a concrete analysis of sampling error, it is imperative to express $\delta_k$ in terms of $\{\delta_k^h\}_{h\in[H]}$. The expression relies on the following proposition:
\begin{proposition}\label{prop:tv_chain rule}
Suppose that we have four random variables $X_i$ and $Y_i$, $i=1,2$. Denote the conditional distribution of $Y_i$ given $X_i=x$ by $p_i(\cdot|x)$, $i=1,2$. Let $q_i$ be the joint distribution of $(X_i, Y_i)$ and $q^X_i$ be the distribution of $X_i$, $i=1,2$. If $\sup_x \mathrm{TV}(p_1(\cdot|x), p_2(\cdot|x)) \leq \epsilon < \infty$, then
\begin{align*}
    \mathrm{TV}(q_1, q_2) \leq \mathrm{TV}(q_1^X, q_2^X) + \epsilon.
\end{align*}
\end{proposition}
\begin{proof}[Proof of \Cref{prop:tv_chain rule}]
In this proof, we abuse notation by identifying a measure with its density for convenience. By definition of total variation distance, 
\begin{align*}
    2 \mathrm{TV}(q_1, q_2) & =  \int |p_1(y|x)q_1^X(x) - p_2(y|x)q_2^X(x)|dydx\\
    & \leq \int |p_1(y|x)q_1^X(x) - p_2(y|x)q_1^X(x)|dydx + \int |p_2(y|x)q_1^X(x) - p_2(y|x)q_2^X(x)|dydx\\
    & \leq \int \bigg(\int |p_1(y|x) - p_2(y|x)|dy\bigg) q_1^X(x)dx + \int \bigg(\int p_2(y|x)dy\bigg)|q_1^X(x) - q_2^X(x)|dx\\
    & = 2\int \mathrm{TV}(p_1(\cdot|x), p_2(\cdot|x)) q_1^X(x)dx + \int |q_1^X(x) - q_2^X(x)|dx\\
    & \leq 2\epsilon + 2\mathrm{TV}(q_1^X, q_2^X),
\end{align*}
where for the last equality, we use the fact that for any fixed $x$, $p_2(y|x)$ is a probability density and hence $\int p_2(y|x)dy=1$. This concludes the proof.
\end{proof}
With this proposition in hand, we are ready to prove Proposition~\ref{prop:delta} which we restate here first.

\begin{proposition}\label{prop:delta_restate}
Let $\delta_k^h$ be the sampling error (in the total variation sense) induced by our sampler at step $h$ episode $k$ and $\delta_k$ be defined in section~\ref{sec:feel_good}, $h\in[H]$ and $k\in[K]$. Then $\delta_k \leq \sum_{h=1}^H\delta_k^h$.
\end{proposition}
\begin{proof}[Proof of Proposition~\ref{prop:delta}.]
Let $q_k^{h:H}$ be the joint distribution of $\{q_k^s\}_{h\leq s\leq H}$ and $\Tilde q_k^{h:H}$ be the joint distribution of $\{\Tilde q_k^s\}_{h\leq s\leq H}$. Then by Proposition~\ref{prop:tv_chain rule},
\begin{align*}
    \delta_k = \mathrm{TV}(q_k, q_k') \leq \delta_k^1 + \mathrm{TV}(q_k^{2:H}, \Tilde{q}_k^{2:H}).
\end{align*}
Likewise, for $h=2,\ldots,H-1$,
\begin{align*}
    \mathrm{TV}(q_k^{h:H}, \Tilde{q}_k^{h:H}) \leq \delta_{k}^h + \mathrm{TV}(q_k^{h+1:H}, \Tilde{q}_k^{h+1:H}).
\end{align*}
Since $q_k^{H:H} = q_k^H$ and $\Tilde q_k^{H:H} = \Tilde q_k^H$, we have $\mathrm{TV}(q_k^{H:H}, \Tilde{q}_k^{H:H}) = \delta_k^H$. Then we conclude the proof by combining the above inequalities.
\end{proof}
Now let $T$ be the truncation map
\begin{align*}
    T(x) := \min\{x, b\}^+
\end{align*}
used in Algorithm~\ref{Algorithm:FG-LMC}. Then conditional on $\Tilde{Q}_{h+1}^k$, $\Tilde{q}_k^h$ is now given by
\begin{align*}
    \widetilde{q}_k^h = T_{\#}[\widetilde{q}^k_{h, h+1}(\cdot | \widetilde{Q}_{h+1}^k)].
\end{align*}
Since we assume that $Q_{h}^k\in \cQ$ for all $h$ and $k$, we have that $q_k^h = T_{\#}[q_k^h]$. And therefore the data-processing inequality gives
\begin{align*}
    TV(\Tilde{q}_k^h, q_k^h) = TV(T_{\#}[\widetilde{q}^k_{h, h+1}(\cdot | \widetilde{Q}_{h+1}^k)], T_{\#}[q_k^h]) \leq TV(\widetilde{q}^k_{h, h+1}(\cdot | \widetilde{Q}_{h+1}^k), q_k^h),
\end{align*}
which is exactly the sampling error in Proposition~\ref{prop:delta_restate}. And hence the conclusion in the proposition still holds for $\delta_k$ with the truncation error.

\subsection{Proof of Linear MDPs}\label{sec:proof_linear_MDP}
In this section, we prove some properties for linear MDPs.
\begin{proposition}\label{proposition:linear_MDP}
In linear MDPs, the linear function class $\cQ$ satisfies Assumption~\ref{assumption_realizability} and~\ref{assumption_boundedness} with $b=H$. And the decoupling coefficient is bounded by
\begin{align*}
    \decoup \leq 2dH(1 +\ln(2KH)).
\end{align*}
\end{proposition}
\begin{proof}[Proof of \Cref{proposition:linear_MDP}]
Note Assumption~\ref{assumption_realizability} can be verified directly by the definition of linear MDP. This boundedness follows from the fact that $r_h\in[0,1]$ for all $h\in[H]$ and hence $Q_h(x, a) \leq H-h+1 \leq H$ for any $x, a\in \cS \times \cA$. Since $Q_h$ is arbitrary here, we have $b=H$. For the upper bound of the decoupling coefficient, we refer to~\cite[Proposition 1]{dann2021provably}.
\end{proof}

\begin{proposition}\label{prop:linear_embedded}
In linear MDPs, the linear function class $\cQ$ satisfies Assumption~\ref{assumption_completeness}. And given any state $x
\in\mathcal{S}$ and $h\in[H]$, we have the following representation of $Q\in \cQ$: 
\begin{align*}
    Q_h(x,a) - [\cT_h^* Q_{h+1}](x, a) = \langle u_h,  \phi(x,a)\rangle.
\end{align*}
\end{proposition}
\begin{proof}[Proof of \Cref{prop:linear_embedded}]
The linearity of the action-value functions directly follows from the Bellman equation:
\begin{align*}
    Q_h(x,a) = r_h(x,a) + (\PP_h V_{h+1})(x,a) = \langle \phi(x,a), \theta_h\rangle+ \int_{\mathcal{S}} V_{h+1}(x')\langle\phi(x,a), d\mu_h(x')\rangle.
\end{align*}
And likewise
\begin{align*}
    [\cT_h^* Q_{h+1}](x, a) = \langle \phi(x,a), \theta_h\rangle + \int_{\mathcal{S}} \max_{a'\in\mathcal{A}}Q_{h+1}(x', a')\langle\phi(x,a), d\mu_h(x')\rangle.
\end{align*}
Then the completeness follows by defining
\begin{align*}
    Q_h(x,a) = \langle\phi(x,a), \theta_h+ \int_{\mathcal{S}} \max_{a'\in\mathcal{A}}Q_{h+1}(x', a') d\mu_h(x')\rangle.
\end{align*}

And therefore
\begin{align*}
    Q_h(x,a) - [\cT_h^* Q_{h+1}](x, a) = \langle\phi(x,a), u_h\rangle,
\end{align*}
where $u_h = \int_{\mathcal{S}}(V_{h+1}(x') -  \max_{a'\in\mathcal{A}}Q_{h+1}(x', a'))d\mu_h(x')$.
\end{proof}

Next, we restate \Cref{lemma:kappa_linearMDP} and provide proof for it.
\begin{lemma}\label{lemma:kappa_linearMDP_restate}
If the stage-wise priors $p_0^h$ are chosen as $\mathcal{N}(0, \sqrt{d}HI_d)$, then $\kappa(\epsilon) = dHO(\ln(dH/\epsilon))$.
\end{lemma}

\begin{proof}[Proof of \Cref{lemma:kappa_linearMDP}]
By our choice of $p_0^h$,
\begin{align*}
    p_0^h(\cQ_h(\epsilon, Q_{h+1})) = O(\epsilon^d (2\pi \sqrt{d}H)^{d/2}).
\end{align*}
And hence
\begin{align*}
    \ln\left(\frac{1}{p_0^h(\cQ_h(\epsilon, Q_{h+1}))}\right) = O(d\ln(1/\epsilon)+d\ln(dH)) = dO(\ln(dH/\epsilon)).
\end{align*}
And finally,
\begin{align*}
    \kappa(\epsilon) = \sum_{h=1}^H\ln\left(\frac{1}{p_0^h(\cQ_h(\epsilon, Q_{h+1}))}\right) = dH O(\ln(dH/\epsilon)).
\end{align*}
\end{proof}

\section{More Details on Approximate Samplers}\label{section:samplers}

In this section, we provide details of samplers with the target distribution $\mu\propto e^{-L}$ in $\RR^d$, where $L$ is twice continuously differentiable, satisfying
the conditions $m I_d \preccurlyeq \nabla^2 L \preccurlyeq M I_d$ for some $M \geq m > 0$. We also define the condition number of $\mu$ by $\kappa = \frac{M}{m}$.

\subsection{Langevin Monte Carlo (LMC)}
The Langevin Monte Carlo (LMC) algorithm samples from a target distribution $\mu \propto \exp(-L)$, using a discretized version of the continuous-time Langevin diffusion. Given an initial distribution $p_0$, and a step size $\tau > 0$, LMC generates a Markov chain $\{w_n\}_{n=0}^N$, starting from $w_0 \sim p_0$ where $w_n \in \mathbb{R}^d$. At each iteration $(n+1)$, the chain updates its state, $w_n$, using:
\begin{align*}
w_{n+1} &= w_n - \tau \nabla L(w_n) + \sqrt{2\tau} \xi_n,
\end{align*}
where $\xi_n$'s are samples from the d-dimensional standard Gaussian independent of $w_n$. $N$ is the total number of iteration. The convergence of LMC to the target distribution is a well-known result, holding true under specific assumptions~\citep{chewi2022analysis}. For illustration, The following theorem provides a concrete example with sufficient conditions for convergence, which establishes the theoretical foundation for our  analysis.
\begin{theorem}[\cite{vempala2019rapid}]\label{thm:convergence_LMC}
Denote the distribution of $w_n$ by $p_n$. For any $\epsilon\in[0, \kappa\sqrt{d}]$, if we take $\tau = O(\frac{\epsilon^2}{\kappa M d})$, then we obtain the guarantee $TV(p_N, \mu)\leq \epsilon$ after
\begin{align*}
    N = \Tilde{\Theta}\left(\frac{\kappa^2 d}{\epsilon^2}\right)\quad \text{iterations}.
\end{align*}
\end{theorem}

\subsection{Underdamped LMC}
The underdamped Langevin dynamics (ULD) for $(w_t, P_t) \in \mathbb{R}^{2d}$ is driven by the SDE
\begin{align}
\begin{split}
dw_t & = P_t dt, \\
dP_t & = -\nabla L(w_t)dt + \gamma P_tdt + \sqrt{2\beta^{-1}\gamma}dB_t,\\ 
\end{split}
\end{align}
which can be viewed as a second-order Langevin dynamics. 
We can use different discretization schemes to obtain discrete-time algorithms. For ULMC using the scheme~\eqref{ulmc:exact}, we have the following convergence result:
\begin{theorem}[\cite{zhang2023improved}]\label{thm:convergence_ULMC}
For ULMC, assume that  our target distribution satisfies that $\E_{\mu}[\|\cdot\|] = m_1 < +\infty$, $\nabla L(0)=0$ (without loss of generality) and $L(0)-\min L = \Tilde{O}(d)$. Then if we set $\tau = \Tilde{\Theta}(\frac{\epsilon m^{1/2}_1}{M d^{1/2}})$ and $\gamma = \Theta(\sqrt{M})$ with a warm start, the law of the $N$-th iterate of ULMC $p_N$ satisfies
\begin{align*}
    TV(p_N, \mu) \leq \epsilon\quad\text{after}\quad N=\Tilde{\Theta}\left(\frac{\kappa^{3/2}d^{1/2}}{\epsilon}\right)\quad\text{iterations}.
\end{align*}
\end{theorem}

\section{Experiment Details}
In this section, we provide further details on our experiment and implementation.

\subsection{$N$-Chain}\label{sec:nchain}

We use $\phi_{\text{therm}}(s_t) = (\mathbf{1} \{x \le s_t\})$ in $\{0,1\}^N$ as input feature following \citet{osband2016deep}. We further follow the same protocol as in \citet{ishfaq2023provable} for our experiments. For all the baseline algorithms, as well as our proposed algorithms FG-LMCDQN, FG-ULMCDQN and ULMCDQN, we parameterize the Q function using a multi-layer perceptron (MLP). We use $[32,32]$ sized hidden layers in the MLP and $ReLU$ as the activation functions. All algorithms are trained for $10^5$ steps where the experience replay buffer size is $10^4$. The performance of each algorithm is measured by the mean return of the last $10$ test episodes. We use mini-batch size of $32$ and set discount factor as $\gamma_{\text{discount}} = 0.99$. The target network is updated every $100$ steps.

For our proposed algorithms, we do a grid search for the hyper-parameters: learning rate $\tau$, bias factor $a$ in the optimizers,  the temperature $\beta$, the friction coefficient $\gamma$ and the Feel-Good prior weight $\eta$. We list the detailed values of all swept hyper-parameters in Table \ref{tb:nchain}. Following \citet{ishfaq2023provable}, for the adaptive bias  term, we set $\alpha_1 = 0.9$, $\alpha_2 = 0.99$, and $\lambda = 10^{-8}$ in~\eqref{Eq:aulmc}.

\begin{figure*}[h]
    \centering
    \includegraphics[width=\textwidth]{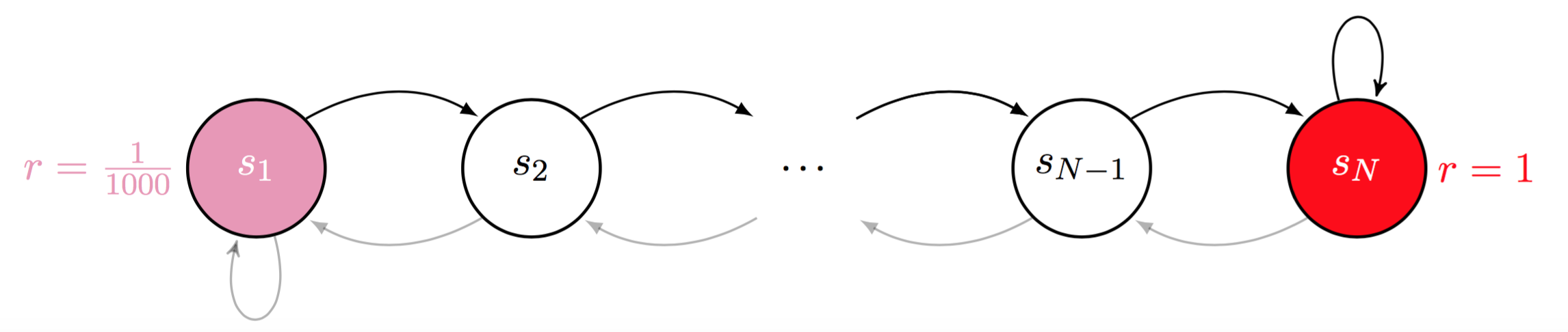}
    \caption{N-Chain environment  \citep{osband2016deep}.}
    \label{fig:nchain}
\end{figure*}

\begin{table}[htbp]
\centering
\begin{tabular}{lcc}
\toprule
{Hyper-parameter} & {Values} \\
\midrule
learning rate $\tau$ & \{$0.01$,  $0.001$\} \\
bias factor $a$ & \{$1.0$, $0.1$, $0.01$\} \\
temperature $\beta$ & \{$10^{12}$, $10^{10}$, $10^{8}$\} \\
update number $J_k$ & \{$4$\}\\
friction coefficient $\gamma$ & \{$1$, $0.1$, $0.01$\}\\
feel-good prior weight $\eta$ & \{$1$, $0.1$, $0.01$\}\\
\bottomrule
\end{tabular}
\caption{The swept hyper-parameter in $N$-Chain experiments.}
\label{tb:nchain}
\end{table}

\subsection{Atari}

\subsubsection{Experiment Setup}\label{sec:atari}
For our Atari experiments, our training and evaluation protocol follows that of \citet{mnih2015human, ishfaq2023provable}.
We implement FG-LMCDQN, FG-ULMCDQN and ULMCDQN using the Tianshou framework \citep{tianshou}.

To maintain a reasonable training time, for all our algorithms FG-LMCDQN, FG-ULMCDQN and ULMCDQN, we set $J_k =1$. Similar to the N-chain experiments, following \citet{ishfaq2023provable}, for the adaptive bias  term, we set $\alpha_1 = 0.9$, $\alpha_2 = 0.99$, and $\lambda = 10^{-8}$ in~\eqref{Eq:aulmc}. We list the detailed values of all swept hyper-parameters in Table \ref{tb:atari}.

\begin{table}[htbp]
\centering
\begin{tabular}{lcc}
\toprule
{Hyper-parameter} & {Values} \\
\midrule
learning rate $\tau$ & \{$0.01$,  $0.001$, $0.0001$\} \\
bias factor $a$ & $\{0.1, 0.01\}$ \\
 temperature $\beta$ & $\{10^{16}, 10^{14}, 10^{12}\}$ \\
update number $J_k$ & \{$1$\}\\
friction coefficient $\gamma$ & \{$10$, $1$\}\\
feel-good prior weight $\eta$ & \{$1$, $10^{-2}$, $10^{-3}$, $10^{-4}$, $10^{-5}$\}\\
\bottomrule
\end{tabular}
\caption{The swept hyper-parameter in Atari games.}
\label{tb:atari}
\end{table}

\subsubsection{Raw Scores in Atari}

In \Cref{table:atari_score_final}, we directly compare the performance of our proposed algorithms FG-LMCDQN, FG-ULMCDQN and ULMCDQN with other baselines by presenting the maximal score obtained by each algorithm in each of the 8 Atari games. The results are averaged over 5 random seeds.

\begin{table}[th]
\centering
\resizebox{\linewidth}{!}{
\begin{tabular}{ccccccccc}
\toprule
${}_{\text{Methods}}\mkern-6mu\setminus\mkern-6mu{}^{\text{Games}}$            & Alien & Freeway & Gravitar & H.E.R.O & Pitfall! & Q*bert & Solaris & Venture \\ \midrule
Human            & 7128  & 30      & 3351     & 30826   & 6464     & 13455  & 12327   & 1188    \\
Random           & 228   & 0       & 173      & 1027    & -229     & 164    & 1263    & 0       \\
\midrule
IQN              & 1691 $\pm$ 155 & 34 $\pm$ 0 & 413 $\pm$ 31 & 11229 $\pm$ 1098 & -5 $\pm$ 4 & 14324 $\pm$ 643 & 1082 $\pm$ 258 & 3 $\pm$ 3 \\
BootstrappedDQN & 1067 $\pm$ 133 & 29 $\pm$ 1 & 312 $\pm$ 50 & 13538 $\pm$ 696 &  -53 $\pm$ 19 & 11786 $\pm$ 1474 &  923 $\pm$ 383 &  12 $\pm$ 12 \\
C51              &  1878 $\pm$ 1878 &  34 $\pm$ 0 &  254 $\pm$ 52 &  \textbf{17500} $\pm$ \textbf{2170} &  -19 $\pm$ 16 & 13394 $\pm$ 1500 & 303 $\pm$ 191 & 31 $\pm$ 31 \\
PrioritizedDQN  & 1799 $\pm$ 202 & 23 $\pm$ 10 & 149 $\pm$ 37 & 14462 $\pm$ 1249 & -46 $\pm$ 16 & 9464 $\pm$ 880 & \textbf{1377} $\pm$ \textbf{423} & 23 $\pm$ 23 \\
NoisyNetDQN     & 1545 $\pm$ 217 & 32 $\pm$ 0 & 124 $\pm$ 47 & 6003 $\pm$ 1903 & -31 $\pm$ 21 & 11046 $\pm$ 1594 & 1373 $\pm$ 361 & 24 $\pm$ 21 \\
AdamLMCDQN      & 1772 $\pm$ 188 & 33 $\pm$ 0 & 799 $\pm$ 34 & 11366 $\pm$ 348 & -6 $\pm$ 4 & 14628 $\pm$ 498 & 938 $\pm$ 189 & \textbf{1326} $\pm$ \textbf{92} \\
ULMCDQN         & 1999 $\pm$ 687 & 34 $\pm$ 0 & 697 $\pm$ 64 & 15201 $\pm$ 3141 & -3 $\pm$ 2 & \textbf{14704} $\pm$ \textbf{794} & 1195 $\pm$ 390 & 1132 $\pm$ 195 \\
FG-LMCDQN         &    \textbf{2380} $\pm$ \textbf{645}   &    23 $\pm$ 10     &    744 $\pm$ 73      &     11576 $\pm$ 202    &     \textbf{0} $\pm$ \textbf{0}     &    14674 $\pm$ 436    &     735 $\pm$ 132    &     1069 $\pm$ 60    \\
FG-ULMCDQN        & 2030 $\pm$ 446    &    \textbf{34} $\pm$ \textbf{0}     &    \textbf{844} $\pm$ \textbf{312}      &     14044 $\pm$ 1220    &    -9 $\pm$ 9      &    14385 $\pm$ 579    &    1278 $\pm$ 851     &     1198 $\pm$ 198    \\ \bottomrule
\end{tabular}
}
\caption{Experiments results on 8 Atari Games. Table \ref{table:atari_score_final} presents the scores in each environment with 50M frames. For each environment, the algorithms perform 5 runs with random seeds. Then we average the scores for each game over 5 runs as the final result with a 95\% confidence interval. We consider 6 baselines: IQN \citep{dabney2018implicit}, Bootstrapped DQN \citep{osband2016deep}, C51 \citep{bellemare2017distributional}, Prioritized DQN \citep{schaul2015prioritized}, NoisyNet DQN \citep{fortunato2017noisy} and AdamLMCDQN \citep{ishfaq2023provable}. The scores for IQN, C51 and Prioritized DQN are taken from DQN Zoo \citep{dqnzoo2020github}. The scores for Bootstrapped DQN, NoisyNet DQN and AdamLMCDQN are taken from \url{https://github.com/hmishfaq/LMC-LSVI} which is the official github repository of \citet{ishfaq2023provable}. We implemented our algorithms ULMCDQN, FG-LMCDQN and FG-UMLCDQN with Tianshou framework \citep{tianshou} in which we also use the double Q method \citep{van2016deep}. Our algorithms have demonstrated advantages in Alien, Freeway, Gravitor and Pitfall compared with baselines.}
\label{table:atari_score_final}
\end{table}

\end{document}